\documentclass[final,1p,times]{elsarticle}

\usepackage{amssymb}
\usepackage{amsmath}
\usepackage{amsthm}

\usepackage{latexsym}
\usepackage{booktabs}
\usepackage{enumitem}
\usepackage{graphicx}
\usepackage{color}
\usepackage{subcaption}
\usepackage{multirow}
\usepackage{dsfont}
\usepackage{makecell}
\usepackage{indentfirst}
\usepackage{placeins}
\usepackage{stmaryrd}

\usepackage{xcolor}
\definecolor{C0}{HTML}{3182ce}
\usepackage[breaklinks=true,colorlinks,bookmarks=false,citecolor=C0]{hyperref}

\DeclareMathOperator{\relu}{ReLU}
\DeclareMathOperator{\rank}{rank}

\newtheorem{theorem}{Theorem}
\newtheorem{lemma}[theorem]{Lemma}

\newtheorem{definition}{Definition}

\newcommand{\BibTeX}{B\kern-.05em{\sc i\kern-.025em b}\kern-.08em\TeX}

\begin{document}
% \setreviewson

\begin{frontmatter}

  %% Title, authors and addresses

  %% use the tnoteref command within \title for footnotes;
  %% use the tnotetext command for theassociated footnote;
  %% use the fnref command within \author or \affiliation for footnotes;
  %% use the fntext command for theassociated footnote;
  %% use the corref command within \author for corresponding author footnotes;
  %% use the cortext command for theassociated footnote;
  %% use the ead command for the email address,
  %% and the form \ead[url] for the home page:
  %% \title{Title\tnoteref{label1}}
  %% \tnotetext[label1]{}
  %% \author{Name\corref{cor1}\fnref{label2}}
  %% \ead{email address}
  %% \ead[url]{home page}
  %% \fntext[label2]{}
  %% \cortext[cor1]{}
  %% \affiliation{organization={},
  %%             addressline={},
  %%             city={},
  %%             postcode={},
  %%             state={},
  %%             country={}}
  %% \fntext[label3]{}

  \title{Bounding Box Anomaly Scoring \\
    for simple and efficient Out-of-Distribution detection}

  %% use optional labels to link authors explicitly to addresses:
  %% \author[label1,label2]{}
  %% \affiliation[label1]{organization={},
  %%             addressline={},
  %%             city={},
  %%             postcode={},
  %%             state={},
  %%             country={}}
  %%
  %% \affiliation[label2]{organization={},
  %%             addressline={},
  %%             city={},
  %%             postcode={},
  %%             state={},
  %%             country={}}

  \author[cea]{Mohamed Bahi Yahiaoui\corref{cor1}}

  \author[ps]{Geoffrey Daniel}

  \author[cea]{Loïc Giraldi}

  \author[cea]{Jérémie Bruyelle}

  \author[uga]{Julyan Arbel}

  \cortext[cor1]{Corresponding author}

  \affiliation[cea]{organization={CEA, DES, IRESNE, DEC, SESC},
    city={Saint-Paul-lez-Durance},
    postcode={13115},
    country={France}}

  \affiliation[ps]{organization={Université Paris-Saclay, CEA, SGLS},
    city={Gif-sur-Yvette},
    postcode={91190},
    country={France}}

  \affiliation[uga]{organization={Université Grenoble Alpes, Inria, CNRS, Grenoble INP, LJK},
    city={Grenoble},
    postcode={38000},
    country={France}}

  %% Abstract
  \begin{abstract}
    %   \small
    Out-of-distribution (OOD) detection aims to identify inputs that differ from the training distribution in order to reduce unreliable predictions by deep neural networks. Among post-hoc feature-space approaches, OOD detection is commonly performed by approximating the in-distribution support in the representation space of a pretrained network. Existing methods often reflect a trade-off between compact parametric models, such as Mahalanobis-based scores, and more flexible but reference-based methods, such as k-nearest neighbors. Bounding-box abstraction provides an attractive intermediate perspective by representing in-distribution support through compact axis-aligned summaries of hidden activations. In this paper, we introduce Bounding Box Anomaly Scoring (BBAS), a post-hoc OOD detection method that leverages bounding-box abstraction. BBAS combines graded anomaly scores based on interval exceedances, monitoring variables adapted to convolutional layers, and decoupled clustering and box construction for richer and multi-layer representations. Experiments on image-classification benchmarks show that BBAS provides robust separation between in-distribution and out-of-distribution samples while preserving the simplicity, compactness, and updateability of the bounding-box approach.
    %   \normalsize
  \end{abstract}

  %%Graphical abstract
  %   \begin{graphicalabstract}
  %     %\includegraphics{grabs}
  %   \end{graphicalabstract}

  %%Research highlights

  %% Keywords

\end{frontmatter}

%% Add \usepackage{lineno} before \begin{document} and uncomment 
%% following line to enable line numbers
%% \linenumbers

%% main text
%%

%% Use \section commands to start a section
\section{Introduction}

\label{sec:introduction}

Over the past decade, deep learning has become the dominant paradigm in machine learning, with especially strong impact in computer vision and natural language processing \citep{lecun2015deep,vaswani2017attention}. As these data-driven systems transition from benchmarks to deployment in decision and safety-critical settings, ranging from healthcare and industrial control to energy infrastructure, the need for trustworthy artificial intelligence becomes increasingly salient~\citep{perez2024artificial,kaur2022trustworthy}.

%A central capability in this context is out-of-distribution (OOD) detection, which aims to identify inputs that deviate from the training distribution and for which model predictions may be unreliable.

One unreliable behavior of deep learning models is that they always produce predictions, possibly with high-confidence, even when presented with unfamiliar or nonsensical inputs. Out-of-distribution (OOD) detection aims to identify samples that differ substantially from the data distribution used to train the model. Such inputs may arise from noise, distribution shift, or novel contexts not observed during training, forcing the model to extrapolate and potentially yielding unreliable predictions. Consequently, developing effective OOD detection methods has become an important research focus in the deep learning community~\citep{miyaigeneralized,lu2025out}.

Previous works on OOD detection can be broadly grouped into two categories. The first category modifies training to encourage OOD awareness, for instance through auxiliary objectives or uncertainty-aware formulations~\citep{devries2018learning,hendrycks2019selfsupervised,nalisnick2019generative,gal2016dropout,malinin2018prior,laurent2023packed}. The second category comprises post-hoc methods, which attach a scoring rule to an already trained model, using its outputs (e.g., softmax scores, logits, energies) or internal representations, to separate in-distribution (InD) from OOD inputs without retraining or architectural changes. Owing to their practicality, post-hoc detectors have attracted substantial attention, yielding a rich set of scoring functions. Early approaches derived OOD scores from softmax probabilities or logits~\citep{hendrycks2017baseline,liang2018enhancing,liu2020energy}. More recent methods further improve separability by applying test-time transformations or calibration and by adjusting intermediate activations or parameters to mitigate overconfident predictions on OOD inputs~\citep{Wang_2022_CVPR,djurisic2023extremely,Liu2023GEN,xu2024scaling}.

Among post-hoc approaches, a promising family of methods relies on distances or similarity metrics computed in the feature space of a network relative to in-distribution training data~\citep{lee2018simple,zhang2023outofdistribution,sun2022nearest}. The key hypothesis is that InD inputs concentrate in a “typical” region in the representation space, whereas OOD inputs tend to fall farther from it. Accordingly, these methods approximate the InD support in feature space, or at least a high-density region of that support, and assign higher OOD scores to test samples whose representations deviate from this approximation. They are generally model-agnostic and often transfer across tasks and architectures because they operate on intermediate activations rather than task-specific outputs. A central design choice in feature-space OOD detection is how the in-distribution support is approximated: compact parametric models can be efficient but restrictive, whereas highly local reference-based methods can be expressive but less scalable. Bounding-box abstraction provides an appealing intermediate form of geometric support representation through compact summaries of internal activations.

Bounding-box abstraction~\citep{Henzinger}, originally introduced for runtime monitoring and novelty detection, summarizes in-distribution data using compact axis-aligned hyperrectangles in a neural network representation space. These hyperrectangles approximate the support of the training data and can be used to determine whether a test sample internal activations remain within the learned region. As such, bounding-box abstraction fits naturally within post-hoc feature-space OOD detection. While attractive for its interpretability and efficiency, the original instantiation monitors only a single layer and yields only a hard accept/reject decision. Moreover, it couples clustering and box construction in the same representation, which can restrict the monitor to a narrow subset of layers or features.

\paragraph{Contributions}
In this work, we propose Bounding Box Anomaly Scoring (BBAS), a distance-based OOD detection approach that leverages bounding box abstraction to detect OOD data. BBAS quantifies deviations in internal network states between test inputs and training data via an anomaly score, leveraging the geometric structure induced by neuron activations. In our evaluation, BBAS achieves strong OOD detection performance across multiple benchmarks and architectures.
The main contributions are:
\begin{itemize}[nolistsep, leftmargin=.5cm]
  \item We extend bounding-box-based OOD monitoring by incorporating convolutional layer representations in addition to the final fully-connected layer.
  \item We introduce a decoupled clustering and abstraction scheme: clustering is performed on a selected subset of features, while bounding boxes are constructed in a richer representation. This enables multi-layer monitoring while mitigating high-dimensional clustering issues.
  \item We propose an OOD scoring mechanism leveraging neuron exceedance, enabling graded novelty assessment and calibration of detector sensitivity.
\end{itemize}

\paragraph{Outline} This paper is structured as follows. Section~\ref{sec:problem} formalises OOD detection and introduces the scoring-based evaluation setting. Section~\ref{sec:related} reviews distance-based post-hoc detectors and situates our approach with respect to bounding box abstraction, including a detailed comparison and discussion of the design choices addressed by our method. Section~\ref{sec:bbas} presents BBAS and provides an interpretation of the abstraction. Section~\ref{sec:experiments} reports an empirical study of the design choices and compares BBAS with prior work on the unified OpenOOD benchmark~\citep{Zhang2023OpenOODVE}. Finally, Section~\ref{sec:discussion} discusses the flexibility and limitations of the approach.

\section{Out-of-Distribution problem setting}
\label{sec:problem}

This paper examines OOD detection challenges mainly in the context of multiclass classification. Let \( \mathcal{X} \) denote the input space, and \( \mathcal{Y} \) denote the label space, where \( \mathcal{Y} = \{1, 2, \dots, K\} \) for a classification task with \( K \) classes. The training dataset \( \mathcal{D} = \{(x_i, y_i)\}_{i=1}^n \) consists of $n$ i.i.d. samples drawn from the training distribution.  A neural network $\mathcal{N}_{\theta}: \mathcal{X} \to \mathbb{R}^{K}$ parameterised by \( \theta \) is trained on samples from \( \mathcal{D} \) to produce a logit vector as output for classification tasks.

OOD detection is a critical task in machine learning, aiming to identify whether a given input sample belongs to the distribution on which a model was trained or deviates from it. This problem is framed as a binary classification task, where each input \(x \in \mathcal{X}\) is classified as either in-distribution (InD) or out-of-distribution (OOD).

To effectively perform this classification, a scoring function \(S(x)\) can be employed as an intermediate step. The scoring function assigns a real-valued score to each sample, reflecting the likelihood of \(x\) being out-of-distribution. By establishing a threshold \(\tau\), a decision function \(\hat{g}(x)\) can then be defined as:
\[
  \hat{g}(x) =
  \begin{cases}
    1 & \text{meaning } x \sim \textnormal{OOD}, \text{if } S(x) > \tau,    \\
    0 & \text{meaning } x \sim \textnormal{InD}, \text{if } S(x) \leq \tau.
  \end{cases}
\]
Using an intermediate scoring function $S$ facilitates the comparison and benchmarking of multiple OOD detection approaches~\citep{yang2022openood, Zhang2023OpenOODVE} by employing standardised evaluation metrics, such as the Area Under the Receiver Operating Characteristic Curve (AUROC) and the False Positive Rate at 95\% True Positive Rate (FPR95).

\section{Related work}
\label{sec:related}

OOD detection has become a central concern in the deployment of machine learning models, particularly in safety-critical or open-world environments. A common strategy is to attach a post hoc OOD detector to a pretrained network, avoiding retraining or architectural modifications. Many such detectors are distance-based. They approximate the InD data support in a feature space induced by the pretrained network, typically using activations from the penultimate layer. They then assign an OOD score by measuring how far a test sample lies from this support approximation. The quality of the approximation directly affects sensitivity and robustness. Existing families of methods mainly differ in the form of the support approximation they induce in the feature space, from compact parametric representations to highly local reference-based ones, with corresponding implications for detection performance, robustness, and scalability.

\paragraph{Mahalanobis-based detectors}
The Mahalanobis-distance method introduced by~\citet{lee2018simple} approximates class-conditional feature distributions using class-specific means and a shared covariance matrix estimated from training data at a chosen hidden layer. This defines an InD support approximation via ellipsoidal level sets of equal Mahalanobis distance around each class mean. At test time, a sample is passed through the network to extract its feature representation, and its OOD score is computed as the Mahalanobis distance to the nearest class mean.
The method is appealing for its simplicity and closed-form scoring, and performs well in some settings. However, it relies on class-conditional features being reasonably well-approximated by Gaussian structure with comparable covariance; this assumption may not hold in deep feature spaces~\citep{sun2022nearest}. Correspondingly, prior evaluations report settings in which parametric Mahalanobis-style scores can be brittle or fail to separate certain OOD inputs, including synthetic unit-test cases~\citep{pmlr-v202-bitterwolf23a}. These observations motivate alternatives that impose weaker global geometric assumptions or adapt more locally to the feature distribution.

\paragraph{k-Nearest-Neighbor (kNN) OOD detection}
At the other extreme, kNN-based methods impose minimal global distributional structure. The KNN  approach~\citep{sun2022nearest} approximates the InD support by retaining a set of training feature vectors extracted from a fixed hidden layer. Rather than fitting a parametric model, it induces an acceptance region implicitly through local neighborhoods around these reference points. At test time, a sample feature representation is compared against the reference set, and its OOD score is derived from distances to its nearest neighbors.
This method is attractive for its flexibility and ability to capture fine-grained structure in the feature space. A practical limitation is that the approximation is inherently reference-based: it does not reduce to a fixed-size summary independent on the dataset, and updates to the reference distribution typically require maintaining an index over stored features~\citep{qin2019scalable}. These properties can complicate deployment, scalability and lifecycle maintenance as the reference set evolves.

\paragraph{Bounding box abstraction}
Bounding box abstractions~\cite{Henzinger} provide a middle ground between rigid parametric summaries and instance-based representations by adapting interval abstraction~\cite{CousotCousot76-1} to runtime monitoring. The monitor builds class-wise summaries of training representations by fitting a bounding box (i.e., an axis-aligned hyperrectangle) to the activations of a chosen fully connected layer. To avoid a single loose box, it splits the data by class ($D_k$), runs k-means within each class, and fits one box per cluster. At inference, it computes the test activation at the monitored layer and checks whether it lies in any box for the predicted class; if not, it flags novelty. This is a geometric test in representation space, similar in spirit to distance-based scores but using box membership instead of a distance score.

This framework is appealing due to its conceptual simplicity, compact representation, flexibility, and practical updateability~\citep{kueffner2023into, wu2023customizable}. At the same time, the original method is limited by (i) clustering in high-dimensional feature spaces, which restricts the number and type of monitored variables, (ii) a focus on a single fully connected layer, and (iii) a hard accept/reject decision that complicate quantitative comparison with other OOD detection techniques. In the next section, we address these limitations by introducing Bounding Box Anomaly Scoring (BBAS), which defines convolution-adapted monitoring variables, decouples the features used for clustering from those used for box construction to address the restriction due the high dimension, and derives OOD scores from interval exceedances.

% \cite{Henzinger} relies on two assumptions:
% \paragraph{Assumption 1} Class-Dependent Structured Internal Representations
% \paragraph{Assumption 2} OOD Data Deviate from In-Distribution Feature Structure

% Both assumptions are not controversial, as deep neural networks tend to learn structured, class-dependent feature representations in their hidden layers.[expand here]

\section{Bounding Box Anomaly Scoring}
\label{sec:bbas}

%We detect OOD inputs by approximating the InD support in a representation space with class-conditional axis-aligned bounding boxes. At test time, inputs can be scored with regard to the boxes (Section~\ref{sec:anomaly_score}).

\subsection{Monitoring variables}
\label{sec:monitorvars}

Let $\phi:\mathbb{R}^{n_{\textnormal{in}}}\to\mathbb{R}^{N_{\textnormal{var}}}$ be a feature map with an input space of dimension $n_{\textnormal{in}}$ and $N_{\textnormal{var}}$ monitored variables.
Given a set of training points $D$, we compute a bounding box $\mathcal{B}(D,\phi)$ in the feature space by taking a per-coordinate minimum and maximum over $\phi$:

\begin{align}
  \label{eq:BB_def}
  \mathcal{B}(D,\phi)
  := \prod_{j=1}^{N_{\textnormal{var}}}
  \Big[\min_{x\in D}\phi_j(x),\; \max_{x\in D}\phi_j(x)\Big].
\end{align}

This defines an accepted region $\mathcal{R}(D,\phi)$ of inputs whose monitored features remain inside the box:
\begin{align}
  \label{eq:AR_def}
  \mathcal{R}(D,\phi)
  := \{x\in\mathbb{R}^{n_{\textnormal{in}}}\mid \phi(x)\in \mathcal{B}(D,\phi)\}.
\end{align}

The abstraction relies on an explicit representation $\phi(x)$ whose coordinates can be bounded to define a bounding box in the feature space. In fully-connected networks, the considered features are the pre-activation values of the neurons in the intermediate layers.
In CNNs, a convolutional layer can be viewed as the repeated application of a shared linear map across spatially overlapping input patches, producing channel-wise feature responses over an ensemble of local receptive fields. Directly monitoring all spatial activations is high-dimensional and overly sensitive to translation. We therefore construct monitoring variables that summarize channel behavior across spatial locations.

Let \(x \in \mathbb{R}^{n_\text{in}}\) be an input to a CNN, and let \( \mathbf{A}^{(1)}, \mathbf{A}^{(2)}, \dots, \mathbf{A}^{(L)} \) denote the preactivation tensors of the \(L\) convolutional layers that we monitor.
For a convolutional layer \( \ell \in \{1,\dots,L\} \), we write
\[
  \mathbf{A}^{(\ell)} = [\mathbf{A}_1^{(\ell)}, \mathbf{A}_2^{(\ell)}, \dots, \mathbf{A}_{C_\ell}^{(\ell)}]^\top
  \in \mathbb{R}^{C_\ell \times H_\ell \times W_\ell},
\]
where \(C_\ell\) is the number of channels and \(H_\ell,W_\ell\) are the spatial dimensions.

For each monitored convolutional layer \(\ell\), we build three types of monitoring variables as vectors in \(\mathbb{R}^{C_\ell}\):
the activation-fraction vector \(\mathbf{f}^{(\ell)}\) in Equation~\eqref{eq:af},
the channel-wise minimum vector \(\mathbf{m}^{(\ell)}\) in Equation~\eqref{eq:min}, and
the channel-wise maximum vector \(\mathbf{M}^{(\ell)}\) in Equation~\eqref{eq:max},
with entries for channel \(c\in\{1,\dots,C_\ell\}\):
\begin{align}
  \label{eq:af}
  f_{c}^{(\ell)} & \;=\;\frac{1}{H_{\ell}W_{\ell}}\sum_{h=1}^{H_{\ell}}\sum_{w=1}^{W_{\ell}}
  \mathds{1}\!\big(A_{c}^{(\ell)}[h,w] > 0\big),                                                \\
  \label{eq:min}
  m_{c}^{(\ell)} & \;=\;\min_{1 \le h \le H_{\ell},\;1 \le w \le W_{\ell}} A_{c}^{(\ell)}[h,w], \\
  \label{eq:max}
  M_{c}^{(\ell)} & \;=\;\max_{1 \le h \le H_{\ell},\;1 \le w \le W_{\ell}} A_{c}^{(\ell)}[h,w].
\end{align}
For $\relu$ networks, the indicator in Equation~\eqref{eq:af} measures the fraction of spatial locations whose preactivation is positive, i.e., the proportion of locations where the channel is active.

In addition to the per-layer channel summaries above, we include a monitoring variable, denoted \(\mathbf{z}\) taken from the penultimate layer of the classifier, i.e., the representation used to compute logits.

We concatenate the per-layer summaries and normalize those components whose scale depends on activation magnitude (min, max, and $\mathbf{z}$). The activation fractions are inherently scaled between 0 and 1:

\begin{align}
  \label{eq:CNN_phi}
  \phi(x) =
  \begin{aligned}[t]
    \Big[\,
     & \mathbf{f}^{(1)},\;
      \frac{\mathbf{m}^{(1)}}{\lVert \mathbf{m}^{(1)} \rVert_{2}+\varepsilon},\;
      \frac{\mathbf{M}^{(1)}}{\lVert \mathbf{M}^{(1)} \rVert_{2}+\varepsilon},\;
    \dots,                 \\
     & \mathbf{f}^{(L)},\;
      \frac{\mathbf{m}^{(L)}}{\lVert \mathbf{m}^{(L)} \rVert_{2}+\varepsilon},\;
      \frac{\mathbf{M}^{(L)}}{\lVert \mathbf{M}^{(L)} \rVert_{2}+\varepsilon},\;
      \frac{\mathbf{z}}{\lVert \mathbf{z} \rVert_{2}+\varepsilon}
      \,\Big]^\top ,
  \end{aligned}
\end{align}
with a small constant \(\varepsilon>0\) (e.g., \(10^{-12}\)) to prevent division by zero.
Normalizing feature vectors improves distance-based OOD detection methods in computer vision tasks, consistent with observations such as in \citet{sun2022nearest}.

\subsection{Clustering for bounding box construction}
\label{seq:clustering}

%Real-world datasets are rarely well represented by a single axis-aligned box. 
For real-world datasets, using a single axis-aligned box rarely represents efficiently the InD support. We therefore partition the training points into clusters and fit one box per cluster.

Clustering is a preprocessing step whose purpose is to partition training data into subsets that are well-approximated by axis-aligned boxes in the monitored space. However, the direct use of the neurons pre-activations, such as the monitoring variables $\phi(x)$, leads to high-dimensional objects whose clustering is complex.
To address this issue, we introduce a separate clustering representation $\psi(x)$ that produces compact clusters.

\paragraph{Fully connected neural networks}
For fully connected neural networks using $\relu$ activation functions, we use the activation patterns of the network instead of the full hidden features of the network. An activation pattern is the binary vector that represents the neurons that are activated in the neural network. We compare activation patterns with the Hamming distance. We provide a detailed interpretation of this approach in Section~\ref{sec:interpretation}.

\paragraph{Adaptation for CNNs}
In the case of CNNs with $\relu$ activation functions, the geometric interpretation of activation patterns is similar to fully connected neural networks. However, as stated in Section~\ref{sec:monitorvars}, the translational symmetry of CNNs leads us to consider other descriptive monitoring variables than the direct value of the neurons. Similarly, we do not use the direct activation of the neurons as clustering features. Instead, we set $\psi(x)$ to be the concatenation of activation-fraction vectors, defined in Equation~\eqref{eq:af}, from the monitored layers and compare points with Manhattan distance.

\paragraph{Clustering algorithm choice}
To obtain clusters that are well matched to axis-aligned boxes, we use hierarchical clustering with complete linkage (from \texttt{scikit-learn} \cite{scikit-learn}).
Complete linkage merges clusters based on their farthest pair of points, which discourages elongated clusters and yields tighter bounding boxes.
In practice, this yields tighter clusters and, consequently, tighter bounding boxes in the monitored space. In Section~\ref{sec:ablation}, we conduct an empirical study on the impact of the choice of the clustering algorithm.

\paragraph{Clustering application}
In order to build the clusters, we first extract the training data point
that belong to the same class $k$ of the classification task and we compute
the clustering features. It is also possible to exclude data points that are
misclassified by the network or classified with a low confidence probability
to eliminate noisy instances that could degrade OOD detection performance.
We obtain therefore a set of clusters $\mathcal{C}_k$ associated to each class~$k$.

%\vspace{0.5cm}

% After clustering, we compute the full set of per-coordinate bounds in $\phi$ within each cluster. For class $k$, let $\mathcal{C}_k$ be the resulting set of clusters (subsets of training points). Each cluster $\mathbf{c}\in\mathcal{C}_k$ yields a box $\mathcal{B}(\mathbf{c},\phi)$.
% The class-conditional accepted region is the union of these boxes:
% \begin{align}
%  \label{eq:AR_clust}
%  \mathcal{R}_k
%  = \bigcup_{\mathbf{c}\in\mathcal{C}_k}\{x:\phi(x)\in \mathcal{B}(\mathbf{c},\phi)\},
% \end{align}
% and the global accepted region is $\mathcal{R}:=\bigcup_{k=1}^K \mathcal{R}_k$.

\subsection{Bounding box computation}

%In order to build the clusters, we first extract the training data point that belong to the same class $k$ of the classification task and we compute the clustering features. It is also possible to exclude data points that are misclassified by the network or classified with a low confidence probability to eliminate noisy instances that could degrade OOD detection performance. We obtain therefore a set of clusters $\mathcal{C}_k$ associated to each class $k$. 

Once the training data are clustered, we compute the monitoring variables $\phi(x)$ represented by the intermediate features associated to every training data point $x$. Then, we define the bounding boxes for each cluster by computing the bounds of the monitoring variables associated to each training data point in the cluster.
Formally, let us consider a cluster $\mathbf{c} \in \mathcal{C}_k$ associated to a class $k$, the function $\phi_{i}^{(\ell )}$ that maps an input $x$ to the corresponding monitoring variable for the neuron or channel $i$ of the layer $\ell $. The lower bound $L_{\mathbf{c},i}^{(\ell )}$ and upper bound $U_{\mathbf{c},i}^{(\ell )}$ are defined by:
\begin{equation}
  \label{eq:bounds}
  L_{\mathbf{c},i}^{(\ell )} = \min_{x \in \mathbf{c}} \phi_{i}^{(\ell )}(x)~\text{ and }~U_{\mathbf{c},i}^{(\ell )} = \max_{x \in \mathbf{c}} \phi_{i}^{(\ell )}(x).
\end{equation}

Consistently with Equation~\eqref{eq:BB_def}, the complete bounding box for cluster $\mathbf{c}$ built upon these bounds:

\begin{align}
  \label{eq:BB_clust}
  \mathcal{B}(\mathbf{c},\phi)
  := \prod_{l=1}^{N_\mathrm{layers}} \prod_{i=1}^{N_l}
  \Big[L_{\mathbf{c},i}^{(\ell )}; U_{\mathbf{c},i}^{(\ell )}\Big],
\end{align}

where $N_\mathrm{layers}$ is the number of monitored layers and $N_l$ is the number of neurons or channels in the layer $l$.

The class-conditional accepted region for each class $k$ is the union of these boxes:
\begin{align}
  \label{eq:AR_clust}
  \mathcal{R}_k
  = \bigcup_{\mathbf{c}\in\mathcal{C}_k}\{x:\phi(x)\in \mathcal{B}(\mathbf{c},\phi)\},
\end{align}
and the global accepted region is $\mathcal{R}:=\bigcup_{k=1}^K \mathcal{R}_k$.

\subsection{Anomaly scoring via exceedance measures}
\label{sec:anomaly_score}

Considering a new input point $x^*$ predicted in class $\hat{k}$ by the neural network, we want to quantify whether this point is out-of-distribution using the class-conditional bounding boxes built on training data. We define anomaly scores based on \emph{exceedances} based on these bounding boxes.

\paragraph{Exceedance count}
We assign an OOD score by computing the number of times neuron states $\phi(x^*)$ exceed the predefined bounds (Equation~\eqref{eq:bounds}) within each cluster and selecting the minimum count across all clusters associated with the predicted class $\hat{k}$, denoted as $\mathcal{C}_{\hat{k}}$.
Formally, the anomaly score $\mathrm{AS\mbox{-}EC}$ for a specific input $x^*$ associated to a predicted class $\hat{k}$ is defined as:

\begin{equation}
  \label{eq:EC}
  \mathrm{AS\mbox{-}EC}(x^*,k) = \min_{\mathbf{c} \in \mathcal{C}_{k}}
  \sum_{\ell  = 1}^{N_\mathrm{layers}} \sum_{i = 1}^{N_{\ell}}
  \mathds{1}\!\Big(\phi_{i}^{(\ell )}(x^*) \notin [L_{\mathbf{c},i}^{(\ell )},U_{\mathbf{c},i}^{(\ell )}]\Big).
\end{equation}
This score is $0$ when $x^*$ is compatible with at least one cluster box, and it increases with the number of constraints that must be relaxed to include $x^*$. Its maximum value is equal to the number of monitoring variables, when the input exceeds all thresholds of all clusters. The idea behind this heuristic is that the total volume of the set delimited by the box becomes bigger as fewer features define it.

\paragraph{Exceedance distance}
This quantity captures the \emph{magnitude} of violations by computing a distance-to-interval score:
\begin{align}
  \label{eq:dist_to_int}
  d_{\mathbf{c},i}^{(\ell)}(x^*) =
  \begin{cases}
    L_{\mathbf{c},i}^{(\ell)} - \phi_{i}^{(\ell)}(x^*), & \text{if } \phi_{i}^{(\ell)}(x^*) < L_{\mathbf{c},i}^{(\ell)}, \\
    \phi_{i}^{(\ell)}(x^*) - U_{\mathbf{c},i}^{(\ell)}, & \text{if } \phi_{i}^{(\ell)}(x^*) > U_{\mathbf{c},i}^{(\ell)}, \\
    0,                                                  & \text{otherwise}.
  \end{cases}
\end{align}
Aggregating these distances yields a point-to-box distance (a distance to a set). The exceedance-distance score is defined analogously to Equation~\eqref{eq:EC} by selecting the closest cluster:
\begin{equation}
  \label{eq:ED}
  \mathrm{AS\mbox{-}ED}(x^*,k) = \min_{\mathbf{c}\in\mathcal{C}_k}
  \sum_{\ell = 1}^{N_\mathrm{layers}} \sum_{i = 1}^{N_\ell} d_{\mathbf{c},i}^{(\ell)}(x^*).
\end{equation}

This score is a natural choice in the \emph{representation space} induced by the network. However, because it is defined after the mapping $x \mapsto \phi(x)$, $\mathrm{ED}$ does not admit a direct geometric interpretation back in the original input space. In contrast, the exceedance-count score (Equation~\eqref{eq:EC}) retains a simple geometric interpretation which is closely tied to an increase of the corresponding box volume by loosening constraints.

\paragraph{Aggregated anomaly score}
Both anomaly scores defined above are class-conditional and depend on a single class hypothesis. When the classifier assigns high confidence to one class, conditioning on this class is appropriate. However, when predictive uncertainty is high, relying on a single class hypothesis may lead to unreliable anomaly estimates.

To address this issue, we aggregate class-conditional anomaly scores using the network's predictive probabilities. Let $f_\theta(x^\ast)_k$ denote the predicted probability for class $k$. The aggregated anomaly score is defined as
\begin{equation}
  \label{eq:AggAS}
  \mathrm{Agg}(x^\ast) = \sum_{k=1}^{K} f_\theta(x^\ast)_k \cdot \mathrm{AS}(x^\ast, k),
\end{equation}
where $\mathrm{AS}(x^\ast, k)$ denotes the anomaly score using either the exceedance distance or the exceedance count.

This aggregation can be interpreted under the assumption that the class label $k$ is a latent variable and that the anomaly score $\mathrm{AS}(x^\ast,k)$ provides a class-conditional measure of atypicality, the aggregated score corresponds to the expectation of the anomaly under the posterior distribution $p(k \mid x^\ast)$ estimated by the classifier:
\begin{equation}
  \label{eq:AggE}
  \mathrm{Agg}(x^\ast) = \mathbb{E}_{k \sim p(k \mid x^\ast)} \big[ \mathrm{AS}(x^\ast, k) \big].
\end{equation}

From this viewpoint, the aggregation marginalizes the anomaly score over class hypotheses. As a result, inputs that are simultaneously inconsistent with several class-specific bounding boxes and associated with high predictive uncertainty naturally receive higher anomaly scores, improving robustness in ambiguous regions of the input space.

\subsection{Summary of our method}

In summary, our approach consists in two steps. The first one is the definition of the clusters and bounding boxes by using the training data to build the OOD monitor. The second step is the runtime monitoring to detect OOD data. We illustrate the complete method using a schematic representation in Figure~\ref{fig:full_pipeline}.

\begin{figure*}[h!]
  \centering
  % trim={<left> <lower> <right> <upper>}
  \includegraphics[trim={5cm 2.5cm 5cm 2cm},clip,width=\textwidth]{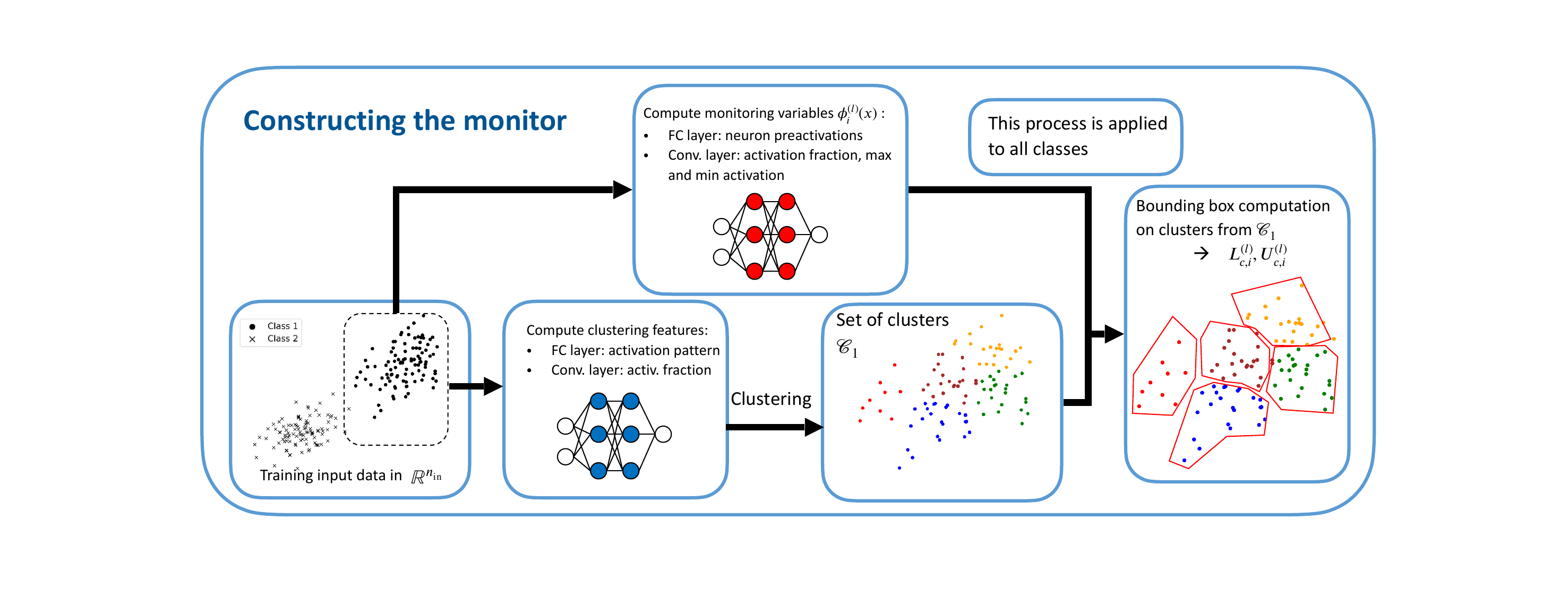}
  \includegraphics[trim={4cm 5cm 5cm 2.5cm},clip,width=\textwidth]{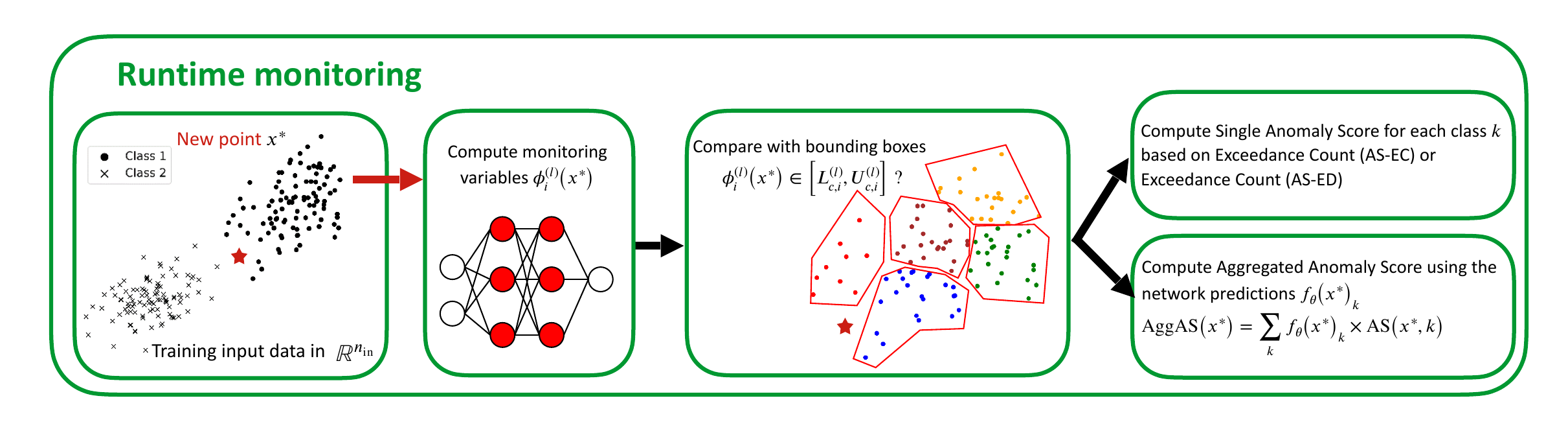}
  \vspace{-.8cm}
  \caption{Complete BBAS pipeline processing: from constructing the monitor to runtime monitoring.}
  \label{fig:full_pipeline}
\end{figure*}

\subsection{Topological interpretation of bounding box abstraction}
% \subsection{Combining Bounding Boxes with Activation Patterns: a Topological Interpretation}
\label{sec:interpretation}

In the case of fully connected networks, the bounding box abstraction has a concrete geometric/topological interpretation that follows directly from a formal construction that we detail in~\ref{app:bb_geometry}. The bounding box abstraction defines a training envelope in the feature space: for a chosen set of hidden features and a set of points, recording their coordinatewise minima and maxima over the reference set has the implication of defining a domain in the input space representing the intersection of a series of sub-level and super-level sets.

With ReLU networks, this has an immediate geometric consequence: the input space is partitioned into activation regions on which the network is affine, as shown in~\ref{sec:act_patterns}. Inside a single activation region, each hidden feature is affine in the input, so the bounding box abstraction becomes a finite set of linear inequalities. In other words, within one activation region, the preimage of the abstraction is a convex polyhedron, and under the mild full-rank condition on the first layer, this polyhedron is bounded, hence a polytope (see Appendix Lemma~\ref{firstlayer} and its extension Lemma~\ref{alllayers}). Globally, the preimage of the bounding box  is just the union of these polytope fragments across the activation regions that intersect the box constraints.

A practical rule of thumb is that the more the box crosses ReLU on/off switches, the more fragments can appear. The \ref{sec:act_patterns} makes this precise by showing how a box can be decomposed by sign configurations, yielding an a priori bound of at most $2^{n_b}$ nonempty fragments when $n_b$ monitored coordinates straddle zero.

Clustering is the mechanism that keeps the abstraction local in the ReLU partition. If we built a single box over all training points, the abstraction would typically span many activation boundaries, which increases the number of activation-region fragments the preimage can break into. By first splitting the data into clusters that are tight in activation space, each per-cluster box tends to cross fewer activation boundaries and therefore intersects fewer activation regions. Concretely, as stated in Section~\ref{seq:clustering}, we cluster activation patterns under the Hamming distance, defined in Equation~\eqref{eq:hamming_dist}, using the agglomerative clustering method with complete linkage, so that patterns within a cluster differ in only a few bits and typically correspond to nearby activation regions. As a result, the axis-aligned min--max envelope fitted to each cluster introduces less empty volume than a global box and yields a sharper operational domain.

We give a concrete illustration of this geometry using the two-moons regression example. We train a ReLU neural network on inputs in $\mathbb{R}^2$ with the target:

\begin{equation}
  f(x)=\sin(\pi x_1)+\sin(\pi x_2).
\end{equation}

We build a bounding box abstraction in three steps: (i) we first group the training samples by identical activation patterns, (ii) we then cluster these groups into $30$ clusters using the Hamming distance between their activation-pattern vectors, and (iii) for each resulting cluster we define a bounding box on the entire hidden-features of the network. This procedure yields $30$ bounding boxes whose preimages are unions of activation-region pieces. Figure~\ref{fig:mlp_two_moons_geom} (left) visualizes where ReLU units switch state and thus how the plane is partitioned into activation regions, while Figure~\ref{fig:mlp_two_moons_geom} (right) shows how the $30$ bounding boxes select subsets aligned with this partition: each cluster-level box covers data that may span several neighboring activation regions.

In this example, we can see that early layers tend to reflect the input-space geometry: they preserve simple notions of closeness that come from the raw input coordinates. As depth increases, representations are increasingly shaped by the function the network has learned, so the geometry becomes aligned with what the model is trying to approximate. In a classification task, we therefore expect late layers to align with the class-relevant representation (samples that the model treats similarly for the decision tend to become close), while early layers mostly align with raw input distances. In image classification, those raw input distances are often not very meaningful semantically, since pixel-space closeness can fail to capture “same object/class”, which is one reason many OOD detection methods focus on later-layer features rather than early ones.

\begin{figure}[t]
  \centering
  \begin{subfigure}[t]{0.49\linewidth}
    \centering
    \includegraphics[width=\linewidth]{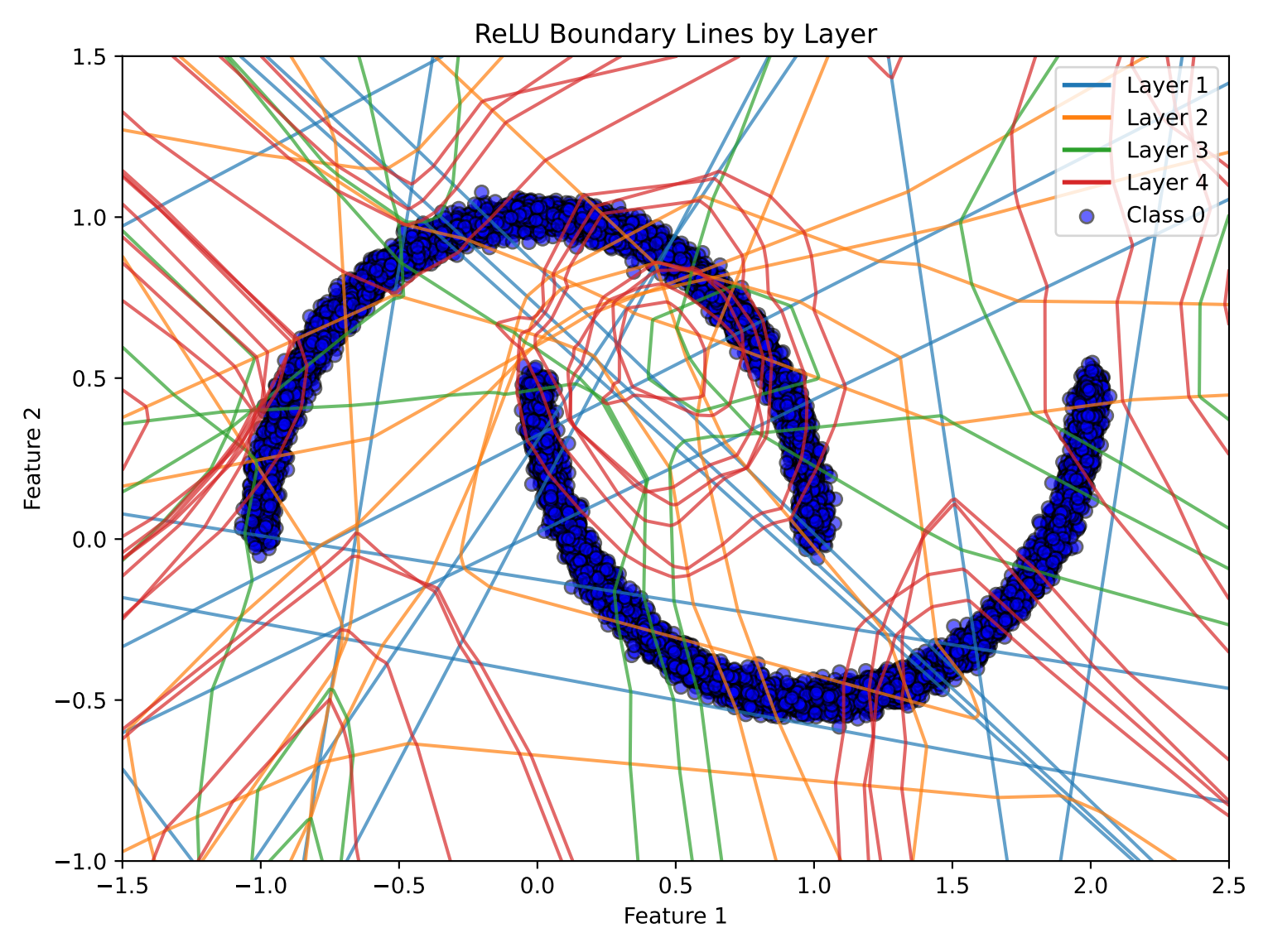}
    \caption{ReLU boundary lines by layer. Each curve is a subset of $\{x:\,a_i^{(\ell)}(x)=0\}$ for some hidden unit $i$ in layer $\ell$; curves are colored by $\ell$. Together, they form the boundaries of activation regions.}
    \label{fig:relu_boundaries_mlp}
  \end{subfigure}\hfill
  \begin{subfigure}[t]{0.49\linewidth}
    \centering
    \includegraphics[width=\linewidth]{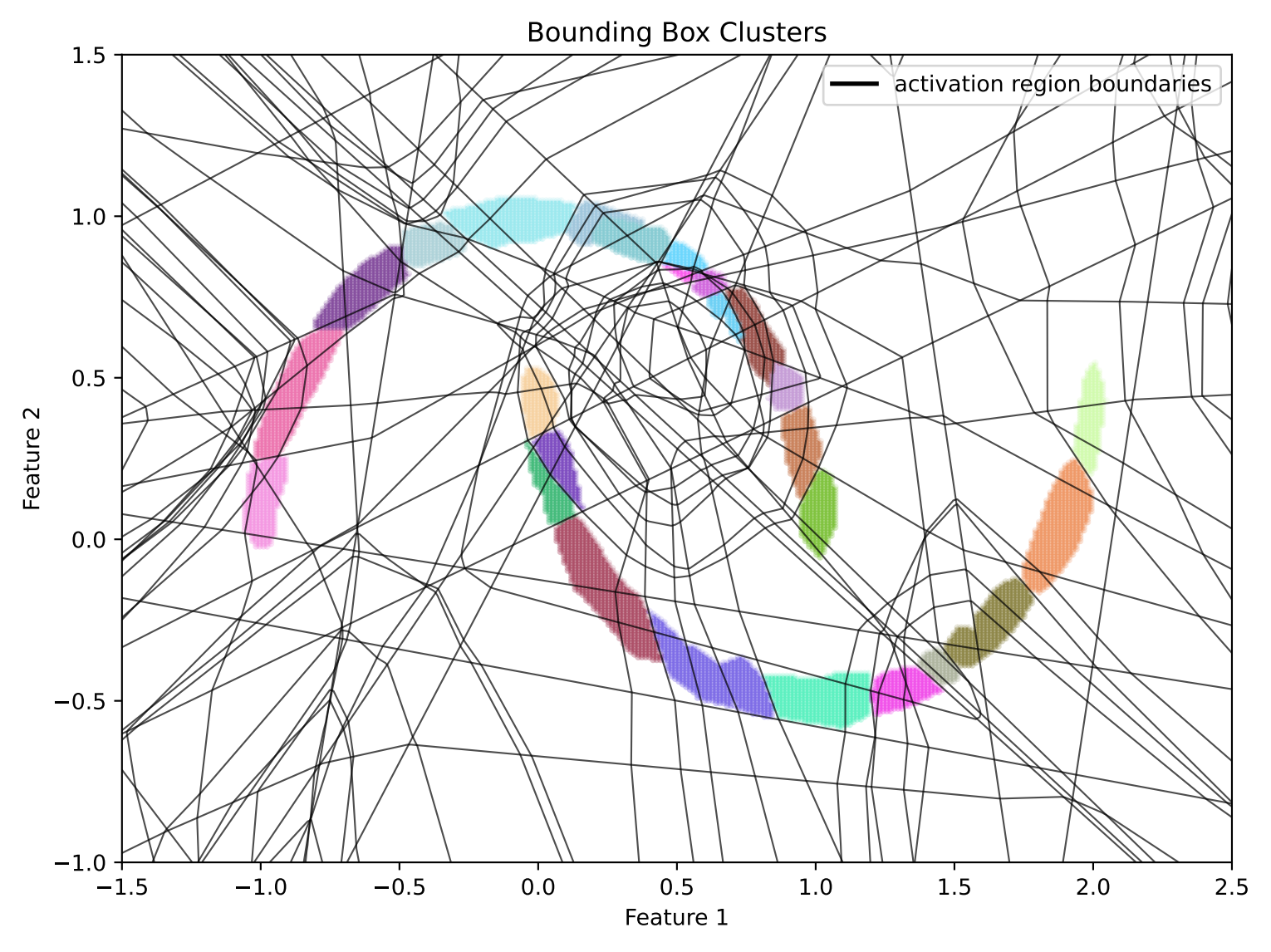}
    \caption{Visualization of the input region induced by $30$ bounding boxes. Each color corresponds to one bounding box and highlights the subset of the input plane whose hidden features fall inside that box. The activation-region boundaries are overlaid in black.}
    \label{fig:bb_clusters_mlp}
  \end{subfigure}
  \caption{Two-moons regression example illustrating the geometric interpretation of activation regions and cluster-wise bounding boxes for a ReLU neural network trained on $f(x)=\sin(\pi x_1)+\sin(\pi x_2)$.}
  \label{fig:mlp_two_moons_geom}
\end{figure}

The same topological picture carries over to CNNs, but the variables we bound must be adapted to the translation invariance of CNNs. A ReLU CNN is still a continuous piecewise-linear map: fixing the sign of every preactivation entry by channel and spatial location makes each convolution and subsequent linear operation affine, so the input space is again partitioned into activation regions that are convex polytopes. As in the fully connected case, a bounding-box monitor defines an operational domain as the preimage of an axis-aligned set under a continuous feature map. Restricted to any activation region, this preimage is obtained by intersecting that region with constraints induced by the monitored features, and the full operational domain is the union of these pieces across all activation regions consistent with the data.

The key difference is that convolutional hidden states are spatial tensors and convolution applies the same local linear map to a sliding collection of input patches, so directly bounding all spatial activations would be both high-dimensional and overly sensitive to translations. This motivates monitoring channel-wise spatial summaries such as activation fractions and channel minima/maxima, which collapse many spatially permuted activation patterns into the same monitored signature while retaining what matters for convolutional processing: how often a channel is active and how extreme its responses become. This coarsens the activation-pattern description by identifying larger equivalence classes of activation regions that share similar channel-level occupancy and envelope behavior, yielding an operational domain that still decomposes along activation regions but is less tied to precise spatial arrangement.

The patch-space view also clarifies what different layers constrain: because each spatial unit at layer $\ell$ depends only on an input patch determined by its receptive field, early-layer monitoring constrains low-level patch statistics (edges, colors, short-range textures), and deep-layer monitoring constrains large-context structure and semantics as receptive fields expand with depth and architectural choices (stride, pooling, dilation). In this sense, a bounding box over channel-wise summaries at layer $\ell$ accepts inputs whose patches, passed through the shared convolutional map over spatial positions, produce channel-wise activation behavior consistent with the training set. If the monitored vectors are additionally normalized, the geometry shifts from magnitude-sensitive constraints to directional ones: bounding an unnormalized vector restricts both scale and channel profile, whereas bounding its normalized version restricts primarily the relative channel pattern (a unit-sphere, angular constraint), making the monitor more invariant to global rescalings but less sensitive to shifts expressed mainly through magnitude. This normalization also breaks the strictly polyhedral picture, since even within a single activation region dividing by an $\ell_2$ norm introduces norm-coupled, generally curved boundaries; the operational domain remains organized by activation regions, but the accepted pieces become activation-region fragments selected by angular constraints rather than purely affine bounds.

\section{Experiments}
\label{sec:experiments}

In this section, we empirically evaluate our bounding-box anomaly scoring framework under the OpenOOD evaluation protocol.
We first analyze the class-dependent structure of our proposed convolutional monitoring variables by visualizing their representations with t-SNE on CIFAR-10.
We then conduct a comprehensive ablation study on CIFAR-10 and CIFAR-100, reporting AUROC separately on Far-OOD (covariate shift) and Near-OOD (semantic shift) splits, to quantify the impact of (i) the choice of monitored layers, (ii) the statistics used to build the bounding boxes, and (iii) the clustering strategy used for box construction.
Finally, we compare against representative post-hoc OOD detection approaches and assess architectural transfer by applying the method to vision transformers, showing that the approach remains competitive beyond CNN backbones.

\subsection{Experimental setup}
\label{sec:experimental_setup}

\paragraph{Benchmark protocol}
For the experimental setting, we use OpenOOD\footnote{\url{https://zjysteven.github.io/OpenOOD/}} ~\citep{Zhang2023OpenOODVE}, a standardized benchmark for OOD detection. OpenOOD is based on image classification datasets: CIFAR-10, CIFAR-100~\citep{cifar-100}, ImageNet-200 and ImageNet-1K~\citep{deng2009imagenet}. For each dataset, it uses a set of OOD data originated from other image datasets. MNIST~\citep{deng2012mnist}, SVHN~\citep{svhn}, Textures~\citep{texture}, Places365~\citep{zhou2017places} are used as Far-OOD for CIFAR datasets; iNaturalist~\citep{DBLP:journals/corr/HornASSAPB17}, Textures~\citep{texture}, OpenImage-O~\citep{DBLP:journals/corr/abs-1811-00982} are used as Far-OOD for ImageNet datasets. Tiny ImageNet~\citep{le_tiny_nodate} and CIFAR-100 are used as NearOOD for the CIFAR-10 dataset; Tiny ImageNet and CIFAR-10 are used as Near-OOD for the CIFAR-100 dataset. SSB-hard~\citep{vaze2022openset} and NINCO~\citep{bitterwolf2023outfixingimagenetoutofdistribution} are used as Near-OOD for both ImageNet datasets. OpenOOD summarizes the performances of different OOD detection methods on each dataset based on the AUROC metric.

\paragraph{Backbone models}
All OOD detectors (including ours) are applied \emph{post-hoc} on \emph{frozen} pre-trained classifiers provided by OpenOOD, allowing for a fair comparison between the different methods. For CIFAR-10, CIFAR-100, and ImageNet-200, we use the official cross-entropy \textbf{ResNet-18} checkpoints released with OpenOOD~\citep{he2016deep,Zhang2023OpenOODVE}.
For ImageNet-1K, we follow the OpenOOD protocol and use a \textbf{ResNet-50} model initialized from the official torchvision weights~\citep{he2016deep,10.1145/1873951.1874254,Zhang2023OpenOODVE}.
Finally, to assess architectural transfer beyond CNNs, we also report results on transformer backbones, namely \textbf{ViT-B/16} and \textbf{Swin-T}, using the official torchvision checkpoints~\citep{dosovitskiy2021imageworth16x16words,liu2021swintransformerhierarchicalvision,10.1145/1873951.1874254}.

\subsection{Class-dependent structure of convolutional features}
\label{sec:class-dependent-features}

To justify the use of channel-wise convolutional statistics for OOD detection, we examine whether these features are class dependent on in-distribution data.
Given a ResNet trained on CIFAR-10, we extract several feature representations from a convolutional layer: (i) per-channel maximum activations, (ii) per-channel minimum activations, and (iii) per-channel activation fractions.
For reference, we also consider the standard \emph{penultimate-layer} features.
We then apply t-SNE to visualize these high-dimensional representations in two dimensions, Figure~\ref{fig:tsne} shows the result.
The samples form compact, well-separated clusters aligned with CIFAR-10 class labels, indicating that even simple channel-wise summaries of convolutional responses retain strong class-specific structure.
This observation supports our choice to include additional features, because in-distribution samples concentrate in class-conditioned regions of feature space, OOD samples are expected to deviate from these regions, enabling effective detection.

\begin{figure}[h]
  \centering

  \begin{subfigure}[t]{0.45\linewidth}
    \centering
    \includegraphics[width=\linewidth, trim=0 0 0 30, clip]{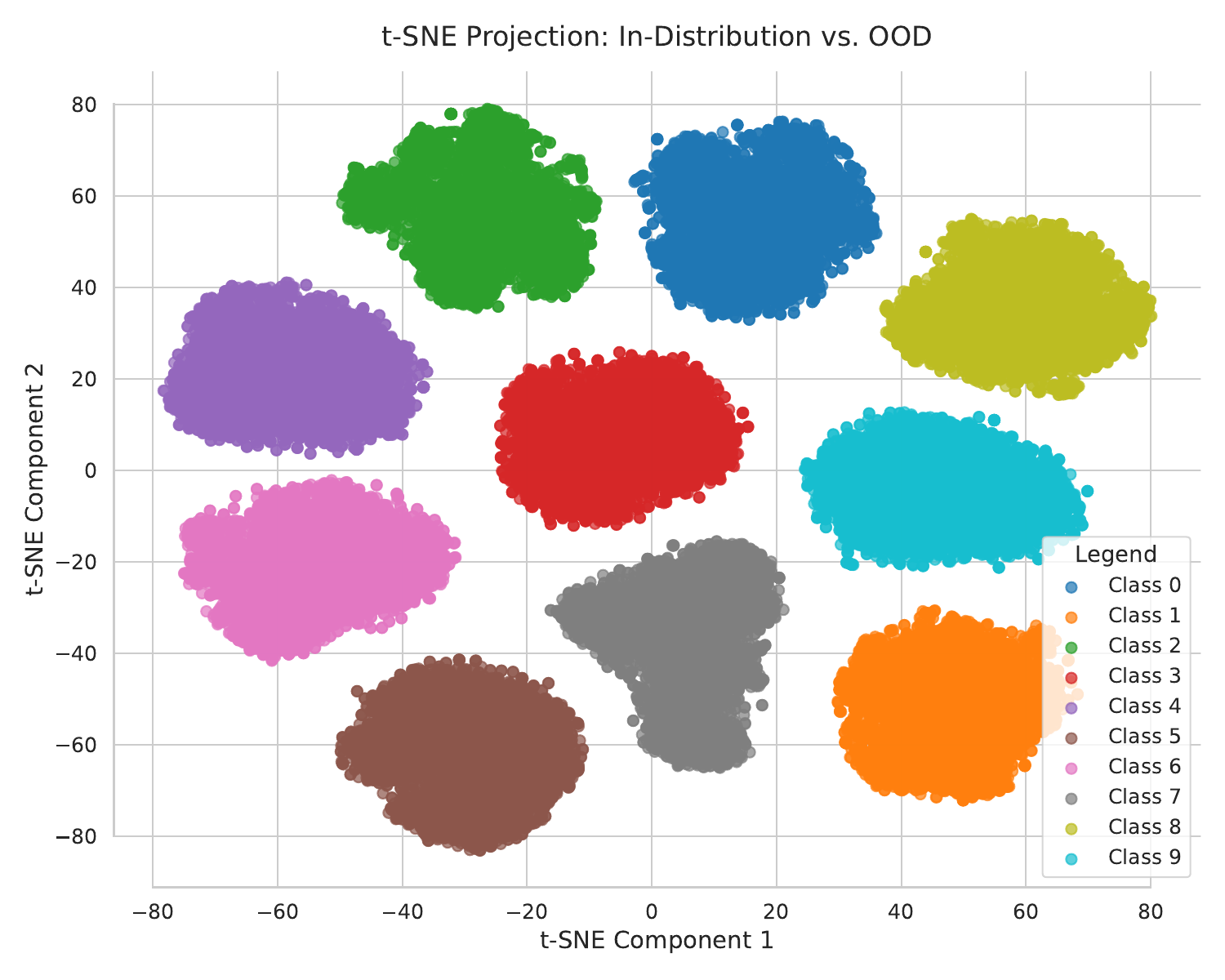}
    \caption{Per-channel maximum activations}
    \label{fig:tsne-max}
  \end{subfigure}
  \hfill
  \begin{subfigure}[t]{0.45\linewidth}
    \centering
    \includegraphics[width=\linewidth, trim=0 0 0 30, clip]{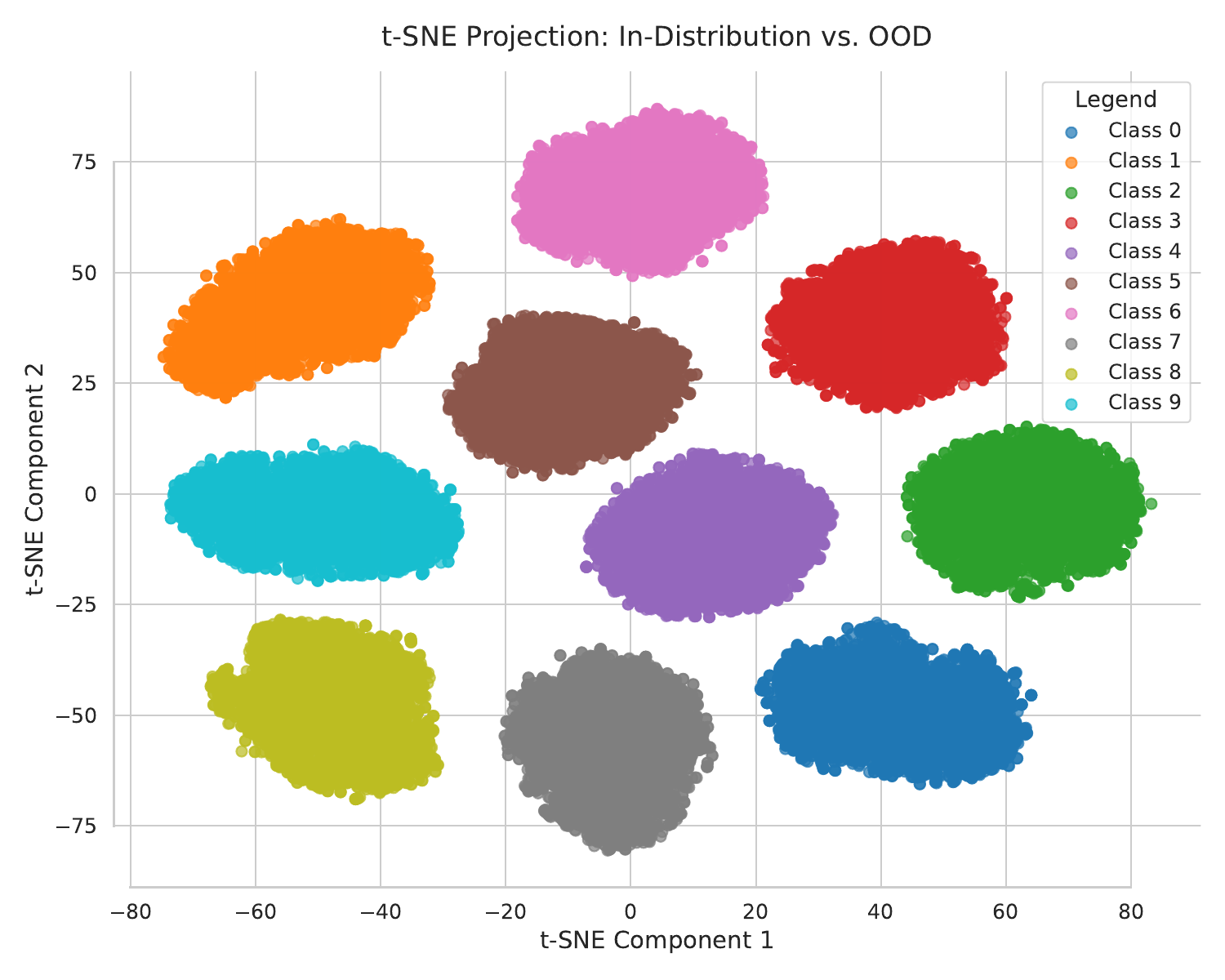}
    \caption{Per-channel minimum activations}
    \label{fig:tsne-min}
  \end{subfigure}

  \vspace{0.5em}

  \begin{subfigure}[t]{0.45\linewidth}
    \centering
    \includegraphics[width=\linewidth, trim=0 0 0 30, clip]{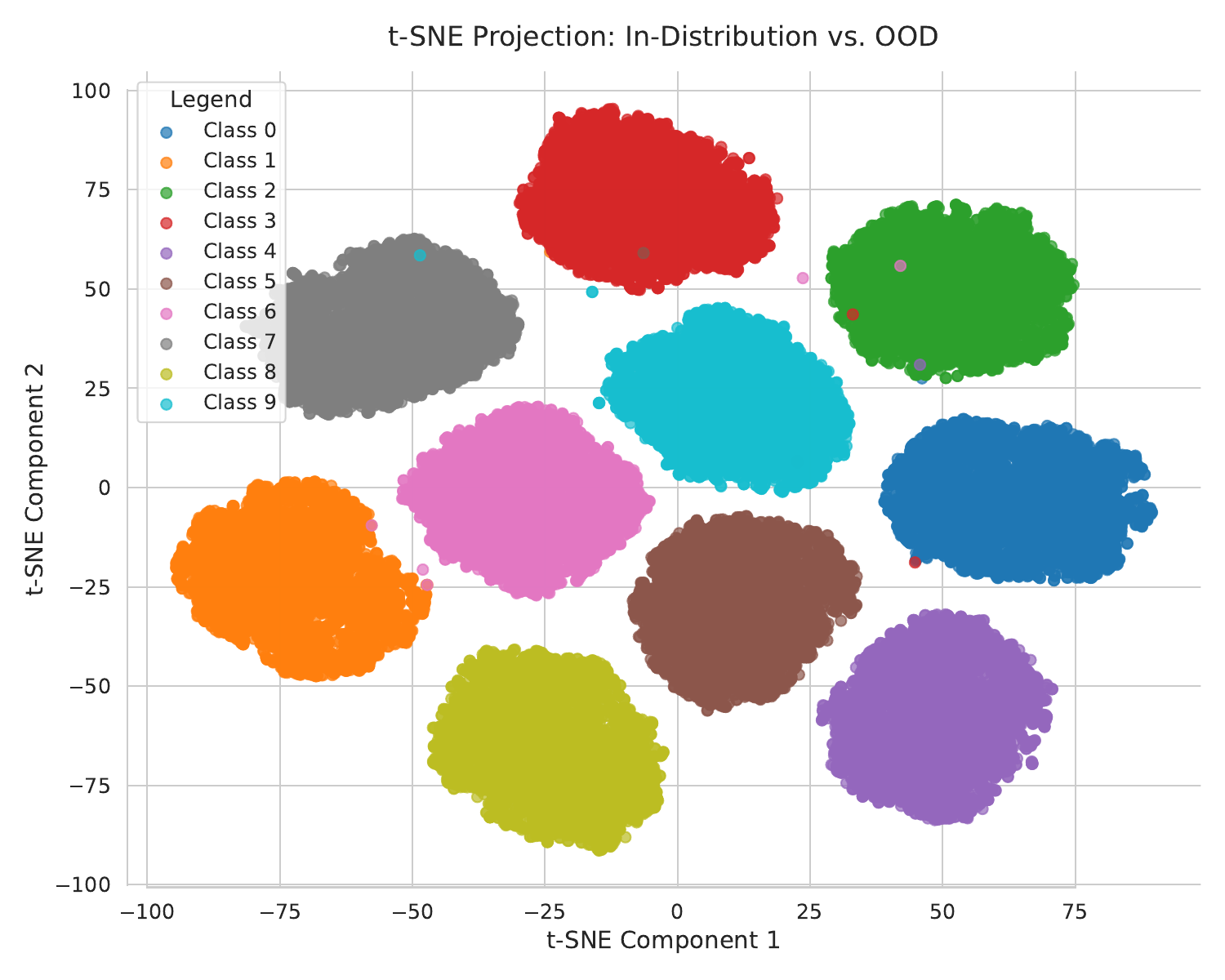}
    \caption{Per-channel activation fractions}
    \label{fig:tsne-aps}
  \end{subfigure}
  \hfill
  \begin{subfigure}[t]{0.45\linewidth}
    \centering
    \includegraphics[width=\linewidth, trim=0 0 0 30, clip]{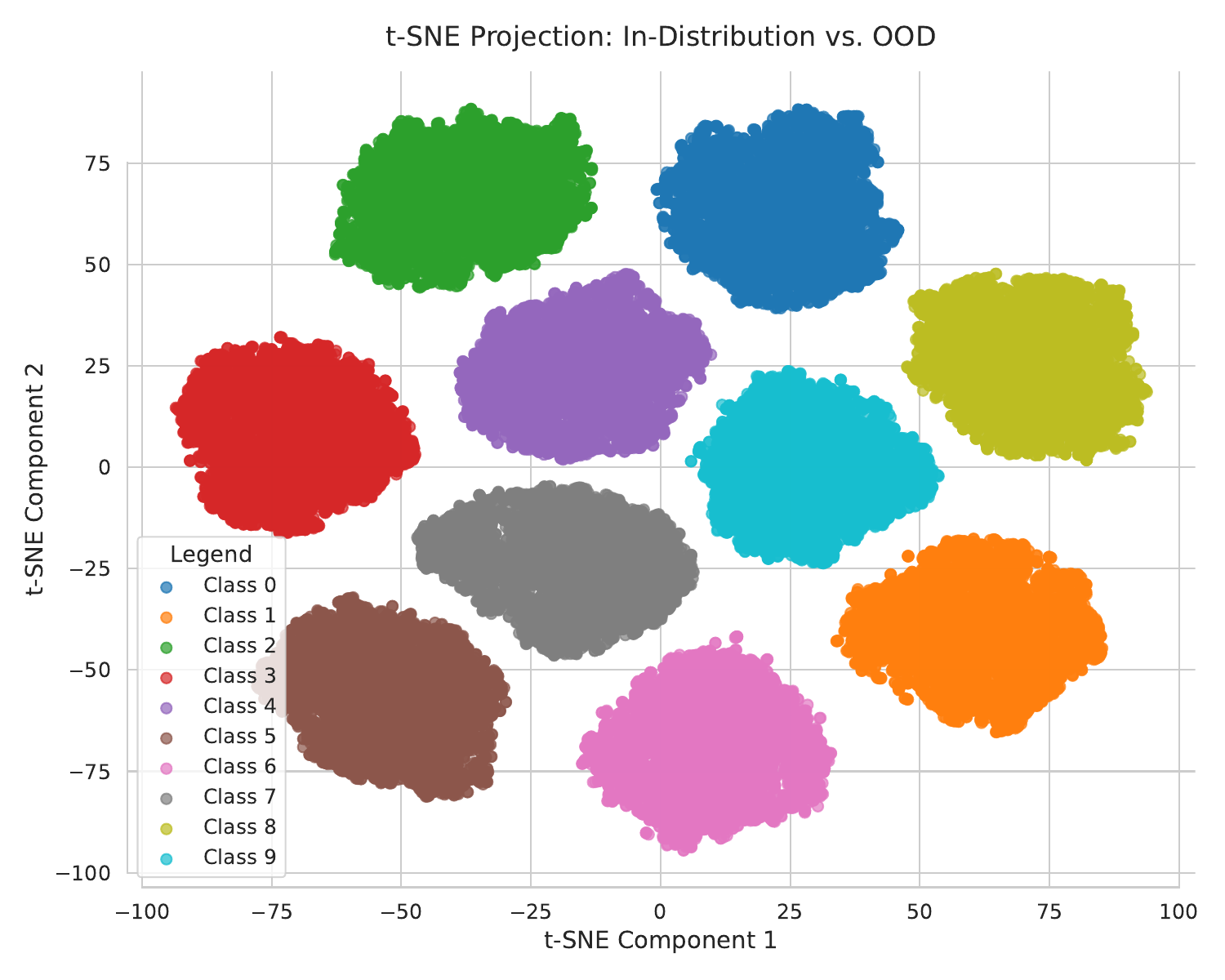}
    \caption{Penultimate-layer features}
    \label{fig:tsne-penultimate}
  \end{subfigure}

  \caption{
    t-SNE visualizations of different feature representations extracted from a ResNet trained on CIFAR-10.
    Across all representations, samples form compact, well-separated class clusters, demonstrating robust class-dependent structure.
  }
  \label{fig:tsne}
\end{figure}

% we first present an ablation study to analyse the impact of different parameter choices on the performance of our method.
\subsection{Ablation study}
\label{sec:ablation}

In this ablation study, we test the claims made in Section~\ref{sec:bbas}  by varying (i) which layers are considered as monitoring variables, (ii) which layer statistics define the bounding box, and (iii) how clustering is performed.
We conduct ablation studies on both Far-OOD (covariate shift) and Near-OOD (semantic shift) datasets derived from the CIFAR-10 and CIFAR-100 benchmarks. All experiments are performed using the same backbone architecture across datasets.
Our method relies on clustering samples based on the activation fraction computed from convolutional features, while the bounding box representation is constructed using the combination of all feature statistics.
Unless stated otherwise, the default configuration extracts features from the last two convolutional layers of each residual blocks in addition to the penultimate layer, uses the combined bounding-box features, and employs activation fraction as the clustering feature, and uses agglomerative clustering with complete linkage. The number of clusters per class is set to the integer part of the square root of the number of retained training samples in that class. At test time, we compute the anomaly score using the exceedance count~\ref{eq:EC} to the corresponding class box.

The results of the ablation study are summarized in Table~\ref{tab:ablation}, where we report AUROC scores separately for Far-OOD and Near-OOD datasets.

\begin{table}[!htbp]
  \centering
  \caption{Ablation study of layer selection, bounding-box feature, and clustering feature choices.
    We report AUROC (\%) on Far-OOD (covariate and semantic shift) and Near-OOD (semantic shift only) datasets for CIFAR-10 and CIFAR-100.
    \textbf{Bold} values indicate the best performance within each ablation group.}\vspace{.1cm}
  \label{tab:ablation}
  \resizebox{\columnwidth}{!}{%
    \begin{tabular}{lcccc}
      \toprule
      \multirow{2}{*}{\textbf{Setting}}     &
      \multicolumn{2}{c}{\textbf{CIFAR-10}} &
      \multicolumn{2}{c}{\textbf{CIFAR-100}}                                                                  \\
      \cmidrule(lr){2-3}\cmidrule(lr){4-5}
                                            & \textbf{Far-OOD} $\uparrow$ & \textbf{Near-OOD} $\uparrow$
                                            & \textbf{Far-OOD} $\uparrow$ & \textbf{Near-OOD} $\uparrow$      \\
      \midrule

      \textbf{a) Layer Choice}              &                             &                              &  & \\
      Penultimate only
                                            & 92.52 $\pm$ 0.45            & 90.04 $\pm$ 0.35
                                            & 83.48 $\pm$ 0.86            & \textbf{79.62 $\pm$ 0.34}         \\
      Last 1 layer
                                            & 92.94 $\pm$ 0.27            & 90.25 $\pm$ 0.36
                                            & 84.74 $\pm$ 0.48            & 77.73 $\pm$ 0.61                  \\
      Last 2 layers (default)
                                            & 94.47 $\pm$ 0.24            & 90.67 $\pm$ 0.42
                                            & 88.03 $\pm$ 0.66            & 77.45 $\pm$ 0.09                  \\
      Last 3 layers
                                            & 95.24 $\pm$ 0.33            & \textbf{90.75 $\pm$ 0.27}
                                            & \textbf{90.01 $\pm$ 0.32}   & 77.05 $\pm$ 0.47                  \\
      Last 4 layers
                                            & \textbf{95.45 $\pm$ 0.07}   & 90.68 $\pm$ 0.41
                                            & 89.91 $\pm$ 0.74            & 76.21 $\pm$ 0.74                  \\

      \midrule
      \textbf{b) Bounding-box Feature}      &                             &                              &  & \\
      Penultimate feature
                                            & 92.52 $\pm$ 0.45            & 90.04 $\pm$ 0.35
                                            & 83.48 $\pm$ 0.86            & \textbf{79.62 $\pm$ 0.34}         \\
      Min activation
                                            & 94.27 $\pm$ 0.12            & 90.54 $\pm$ 0.31
                                            & 86.44 $\pm$ 0.64            & 77.92 $\pm$ 0.42                  \\
      Max activation
                                            & 93.29 $\pm$ 0.08            & 89.72 $\pm$ 0.57
                                            & 87.50 $\pm$ 0.59            & 75.83 $\pm$ 0.31                  \\
      Activation fraction
                                            & \textbf{95.83 $\pm$ 0.24}   & 90.15 $\pm$ 0.54
                                            & 87.18 $\pm$ 0.66            & 72.06 $\pm$ 0.19                  \\
      Combined use (default)
                                            & 94.47 $\pm$ 0.24            & \textbf{90.67 $\pm$ 0.42}
                                            & \textbf{88.03 $\pm$ 0.66}   & 77.45 $\pm$ 0.09                  \\

      \midrule
      \textbf{c) Clustering Feature}        &                             &                              &  & \\
      Min activation
                                            & 94.25 $\pm$ 0.11            & 90.63 $\pm$ 0.36
                                            & 87.88 $\pm$ 0.58            & 78.15 $\pm$ 0.57                  \\
      Max activation
                                            & 94.27 $\pm$ 0.26            & 90.50 $\pm$ 0.36
                                            & 87.47 $\pm$ 0.88            & 77.94 $\pm$ 0.43                  \\
      Penultimate
                                            & 94.07 $\pm$ 0.44            & 90.64 $\pm$ 0.19
                                            & 86.68 $\pm$ 1.19            & \textbf{78.82 $\pm$ 0.47}         \\
      Activation fraction (default)
                                            & \textbf{94.47 $\pm$ 0.24}   & \textbf{90.67 $\pm$ 0.42}
                                            & \textbf{88.03 $\pm$ 0.66}   & 77.45 $\pm$ 0.09                  \\

      \bottomrule
    \end{tabular}
  }
\end{table}

\paragraph{a) Layer choice}

We first study the effect of using features from different convolutional depths, probing the depth-dependent geometry discussed earlier in Section~\ref{sec:interpretation}: shallow layers capture low-level statistics, whereas deeper layers align more with the learned function and decision-relevant structure. We progressively add earlier-layer features to the penultimate layer. On both CIFAR-10 and CIFAR-100, this improves Far-OOD detection up to a point, increasing AUROC under covariate shift. In contrast, Near-OOD performance is unchanged or degrades, indicating that shallow features add little decision-relevant signal and can introduce noise when separating semantically similar but distribution-shifted data. This is expected: Near-OOD datasets share patch-level statistics with the in-distribution data, so low-level cues are not discriminative.

This also reflects the role of monitoring variables: adding early-layer variables can tighten boxes, but may reduce robustness by adding unstable or uninformative constraints. We therefore use the last two convolutional layers, where features are sufficiently task-aligned while avoiding the noisy constraints introduced by shallow layers.

\paragraph{b) Bounding-box feature}

We next analyze the contribution of individual feature statistics used to construct the bounding box.
We compare activation fraction, minimum activation, maximum activation, and the penultimate-layer feature, as well as their combined use.
While single-feature variants achieve competitive performance, combining all feature statistics consistently yields the most robust performance across datasets.
This confirms that different statistics capture complementary aspects of activation behavior, and therefore we adopt their combination in the final framework.

\paragraph{c) Clustering feature}

We analyze the effect of the clustering feature while keeping the layer configuration and combined bounding-box representation fixed.
We find that performance differences across clustering features are relatively small, indicating that the proposed framework is robust to this design choice.
Nonetheless, activation fraction consistently provides strong and stable results, particularly for Far-OOD detection, and is therefore used as the default clustering feature in all experiments.

\paragraph{d) Clustering algorithm and number of clusters}

We further analyze the impact of the clustering algorithm and the number of clusters per class on OOD detection. Keeping all other components fixed, we vary the clustering method (single/complete/average linkage for hierarchical clustering and k-means) and sweep the cluster count. Figure~\ref{fig:cluster-ablation} reports AUROC as a function of the number of clusters.

Single linkage is stable across all experiments. This matches the analysis in Section~\ref{seq:clustering}: single linkage often yields elongated, connectivity-based clusters, and increasing the number of clusters does not substantially reduce the volume of the union of axis-aligned boxes. As a result, the accepted set changes little with the cluster count, and AUROC remains largely unchanged.

Complete linkage is sensitive to the cluster count. In settings where clustering affects performance, AUROC either increases or decreases as more clusters are introduced. On Near-OOD CIFAR-100, where each class has fewer samples, AUROC decreases slightly with more clusters, consistent with increased sensitivity to noise when estimating coordinate-wise extrema from small clusters. In contrast, on CIFAR-10 (both Near-OOD and Far-OOD), AUROC improves as the number of clusters increases, indicating that additional clusters tighten the abstraction.

Average linkage typically lies between single and complete linkage. k-means performs worse than complete linkage, which is consistent with the limitations of centroid-based clustering in high-dimensional feature spaces, where Euclidean distances become less informative.

\begin{figure}[!htbp]
  \centering

  \begin{subfigure}[t]{0.48\linewidth}
    \centering
    \includegraphics[width=\linewidth]{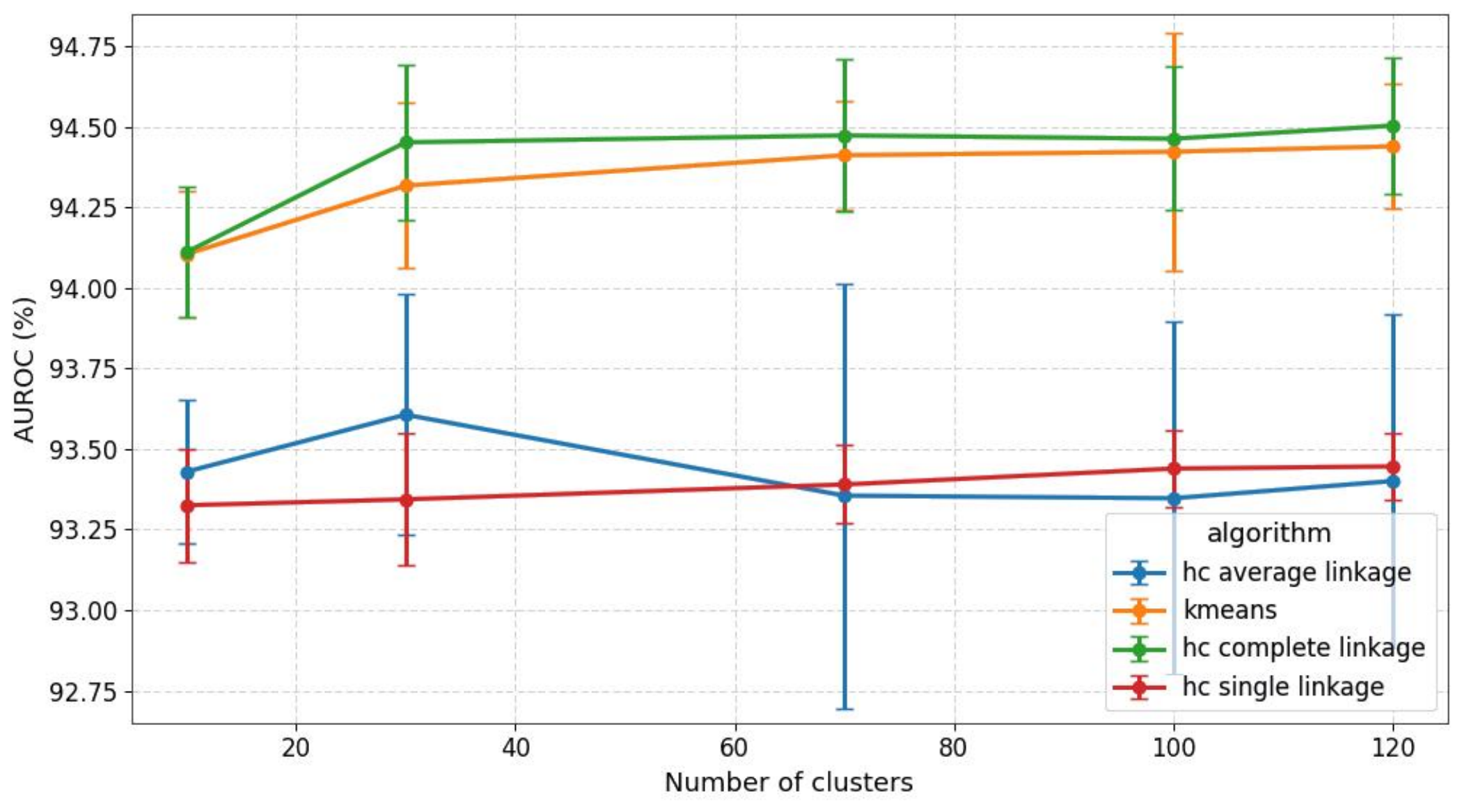}
    \caption{CIFAR-10 Far-OOD}
    \label{fig:cluster-c10-far}
  \end{subfigure}
  \hfill
  \begin{subfigure}[t]{0.48\linewidth}
    \centering
    \includegraphics[width=\linewidth]{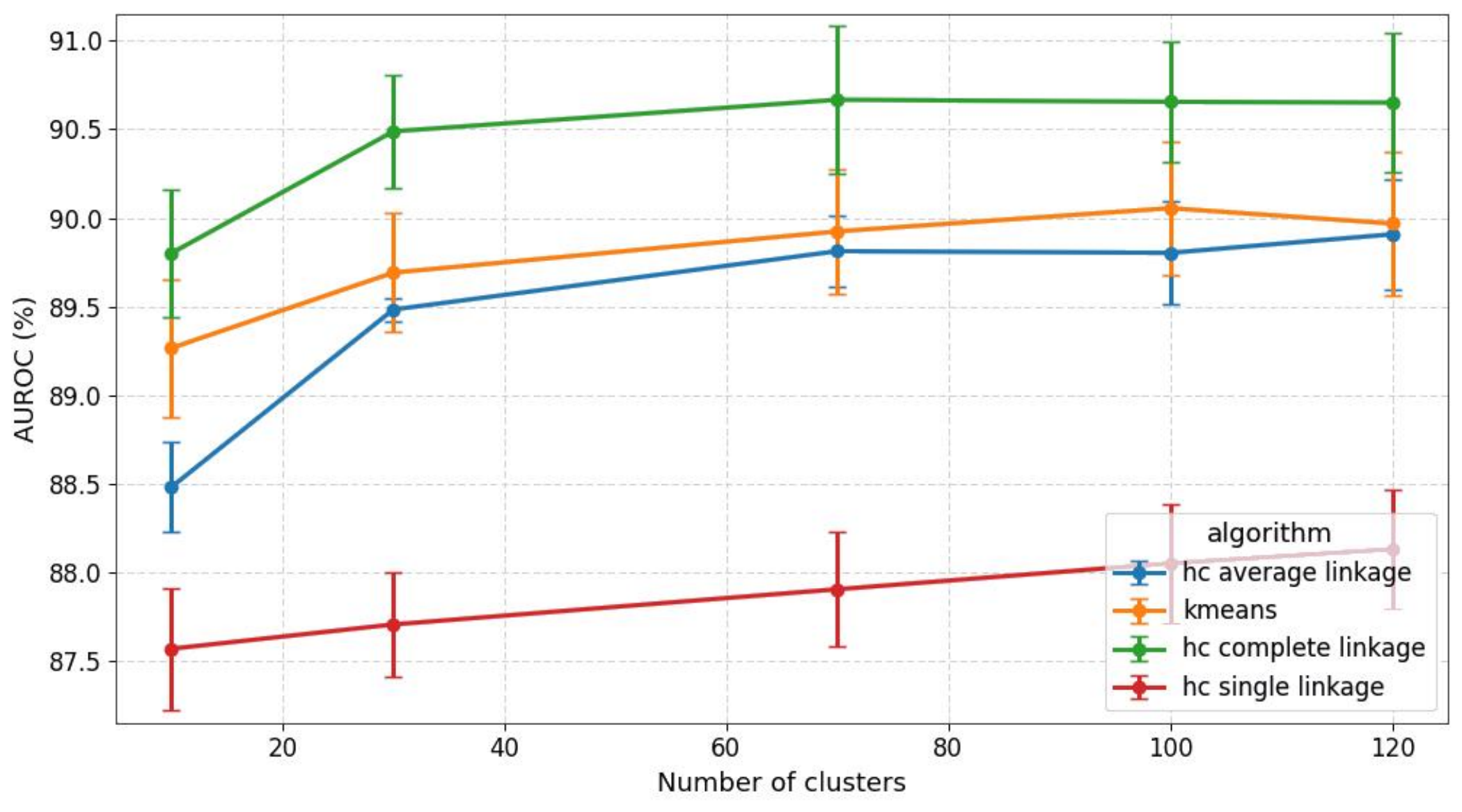}
    \caption{CIFAR-10 Near-OOD}
    \label{fig:cluster-c10-near}
  \end{subfigure}

  \vspace{0.6em}

  \begin{subfigure}[t]{0.48\linewidth}
    \centering
    \includegraphics[width=\linewidth]{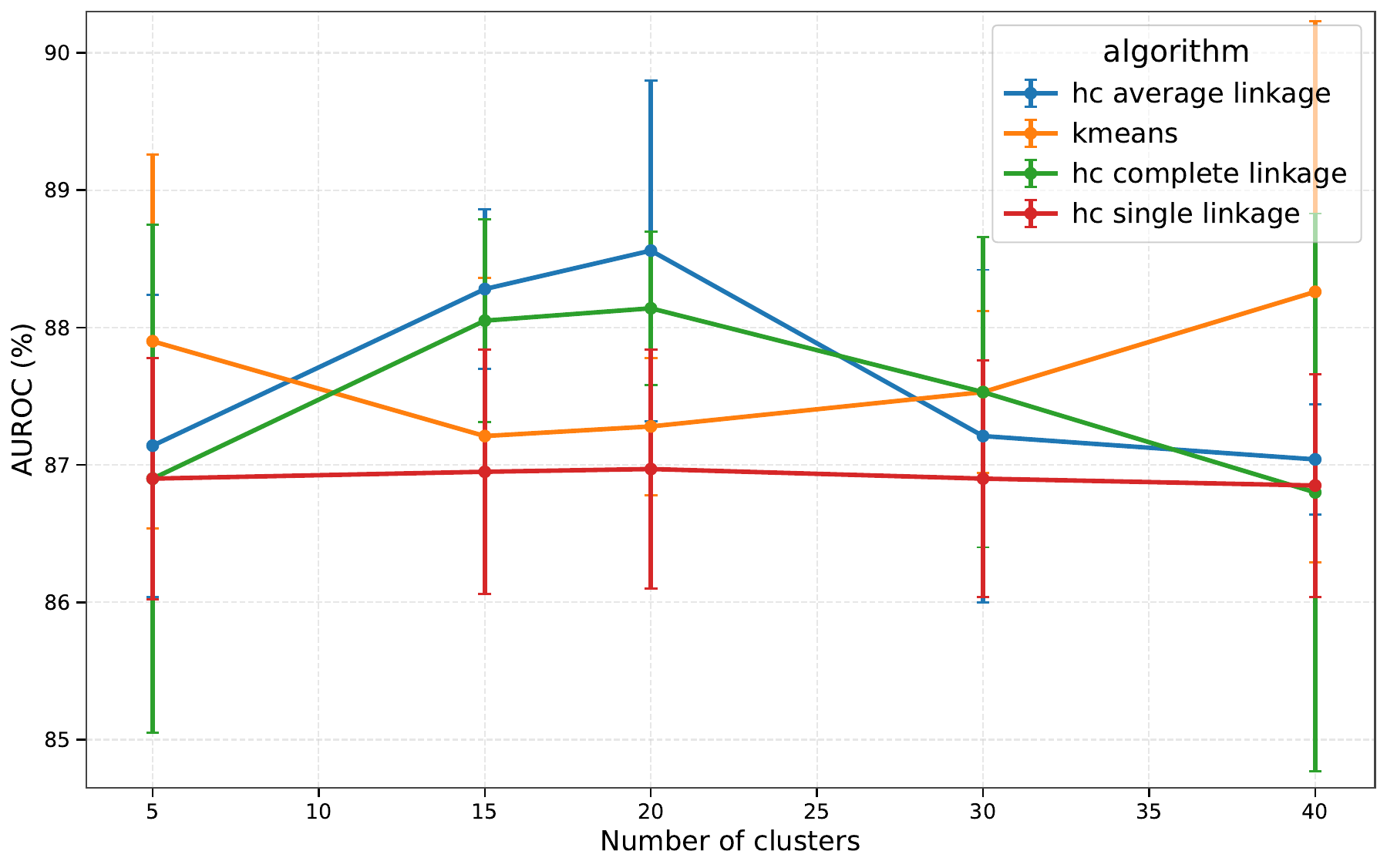}
    \caption{CIFAR-100 Far-OOD}
    \label{fig:cluster-c100-far}
  \end{subfigure}
  \hfill
  \begin{subfigure}[t]{0.48\linewidth}
    \centering
    \includegraphics[width=\linewidth]{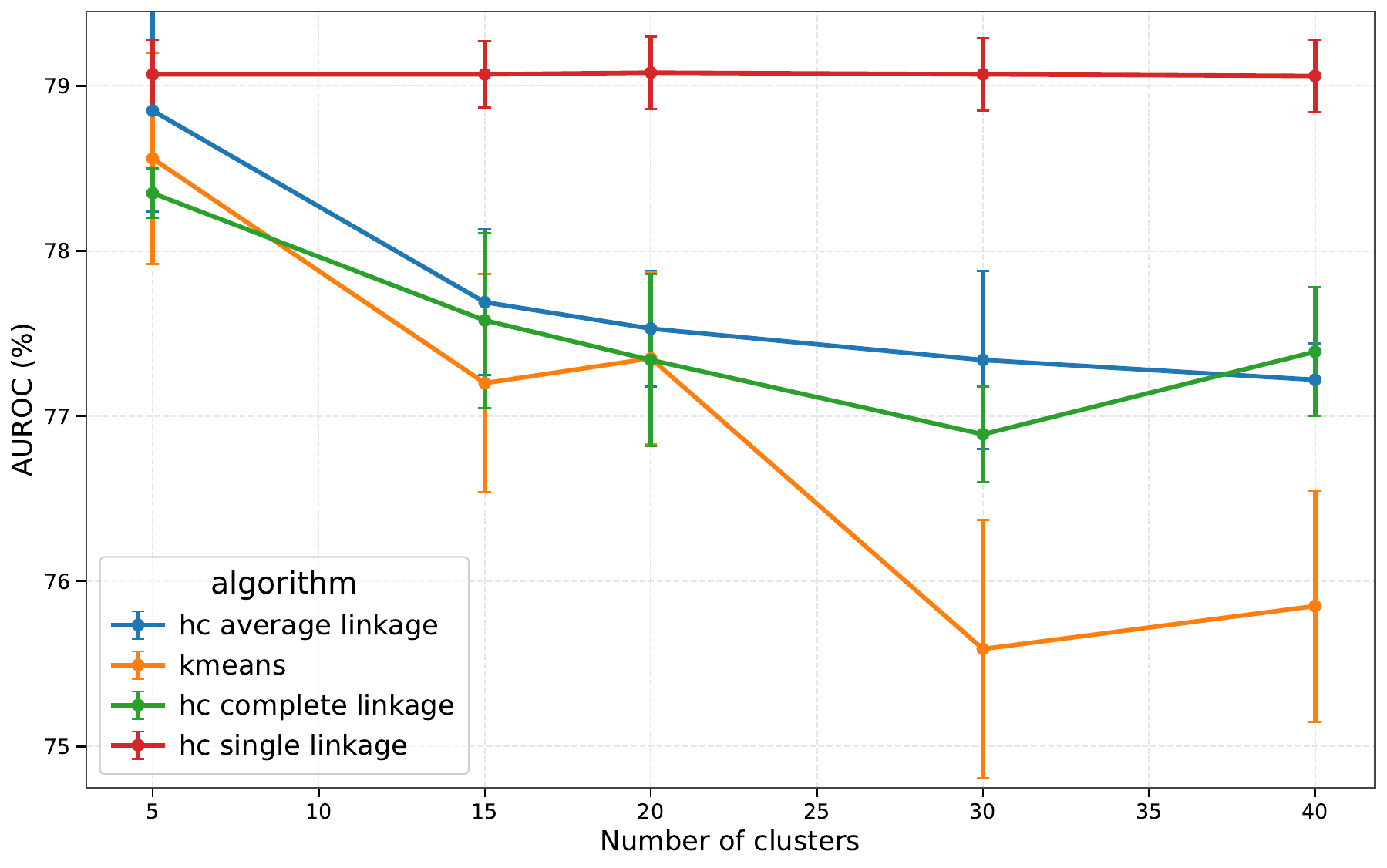}
    \caption{CIFAR-100 Near-OOD}
    \label{fig:cluster-c100-near}
  \end{subfigure}

  \caption{
    Effect of the clustering algorithm and the number of clusters per class on AUROC.
    Agglomerative clustering with complete linkage consistently yields the strongest performance for Far-OOD detection,
    while Near-OOD performance is less sensitive to the clustering choice.
  }
  \label{fig:cluster-ablation}
\end{figure}

\FloatBarrier

\subsection{Comparison with state-of-the-art methods on CNNs}

For these comparisons, we follow the baseline configuration described in Section~\ref{sec:ablation} and report results with both Aggregated and Single Anomaly Scores. On CIFAR-10/100 we use the penultimate representation together with the last two convolutional layers, while on ImageNet-200/1K we pair the same penultimate representation with only the last convolutional layer. We use a set of ready to use pre-trained ResNet models~\citep{he2016deep} as detailed in Section~\ref{sec:experimental_setup}. Tables~\ref{tab:global_perf_farood} and~\ref{tab:global_perf_nearood} summarize the performance of state-of-the-art \textit{post-hoc} methods and report the corresponding results of our proposed BBAS method across a range of Far-OOD and Near-OOD benchmarks, respectively. We also report runtime costs for the main components of the proposed pipeline on CIFAR-100. A detailed breakdown is provided in Table~\ref{tab:comput-cost} in Appendix~\ref{app:runtime_costs}.

For clarity, we report a compact set of widely used \textit{post-hoc} baselines that (i) are representative of distinct scoring principles, (ii) are readily reproducible with the same pre-trained ResNet backbones, and (iii) have been shown to be competitive across common OOD benchmarks. Concretely, we include representative baselines as follows: MSP scores samples by the maximum softmax probability~\citep{hendrycks2017baseline}, and ODIN improves confidence separation using temperature scaling with a small input perturbation~\citep{liang2018enhancing}. VIM leverages feature-space geometry by combining a logit-related term with the residual outside an ID subspace~\citep{Wang_2022_CVPR}. GEN leverages a generative view of feature statistics to produce an OOD score that is less tied to the softmax confidence~\citep{Liu2023GEN}. SHE is a feature-space scoring method designed to improve OOD separation by reshaping or reweighting representations at test time~\citep{zhang2023outofdistribution}.
ASH reduces spurious OOD confidence by modifying internal activations at inference time~\citep{djurisic2023extremely}, and SCALE builds on this activation-shaping idea with a more robust scaling strategy~\citep{xu2024scaling}. We also report distance-based baselines: MDS uses class-conditional Mahalanobis distances and MDSEns is its ensemble variant~\citep{lee2018simple}, RMDS is a stronger Mahalanobis-style scoring variant~\citep{DBLP:journals/corr/abs-2106-09022}, and KNN detects OOD using nearest-neighbor distances in feature space~\citep{sun2022nearest}. This selection covers complementary families (confidence, subspace/residual, activation shaping, and distance-based) while keeping the main tables concise; additional OpenOOD baselines and full per-dataset results are provided in~\ref{sec:detailedres}.

We use the two proposed Anomaly Scored, based on the exceedance count (BBAS-EC, Equation~\eqref{eq:EC}) and the exceedance distance (BBAS-ED, Equation~\eqref{eq:ED}), and their Aggregated version (AggBBAS-EC or -ED, Equation~\eqref{eq:AggAS}). Overall, the exceedance distance score BBAS-ED is typically slightly stronger than the exceedance count score BBAS-EC when used as a single class-conditional distance. However, the most consistent performances come from aggregating over class probabilities using AggBBAS-EC or -ED. Tables~\ref{tab:global_perf_farood} and~\ref{tab:global_perf_nearood} indicate that the benefit of aggregation is highly dataset- and shift-dependent.
On CIFAR-10, aggregation produces only marginal changes, suggesting that predictions are typically unambiguous and conditioning on the top-1 class already yields a stable anomaly score.
On CIFAR-100, aggregation is noticeably more helpful in the Near-OOD setting, where semantically similar classes induce higher class confusion. In contrast, the effect is comparatively smaller on Far-OOD, where the single class-conditional score already provides strong separation.
The largest gains are observed on ImageNet-200/1K for both Far-OOD and Near-OOD, consistent with the larger label space and greater visual diversity increasing the brittleness of single-class conditioning.
Overall, the tables trends support the view that aggregation is most beneficial in regimes with higher class ambiguity and covariate shift severity.
% Aggregation has a relatively minor effect when the classifier is already confident, but leads to much larger gains when class uncertainty is higher.

\paragraph{FarOOD setting}
On Far-OOD, we first compare against general post-hoc approaches (MSP/ODIN, VIM, ASH, and SCALE in Table~\ref{tab:global_perf_farood}). In this regime, the aggregated AggBBAS-EC/ED variants show the most consistently strong performance, ranking best on CIFAR-10, CIFAR-100 and ImageNet-200, while remaining competitive on ImageNet-1K where SCALE (followed by ASH) provides the strongest performance among the listed general baselines.
Next, focusing only on the distance-based methods (MDS/MDSEns, RMDS, KNN), the single anomaly score variants BBAS-EC/ED already place ahead of these baselines on CIFAR-10, CIFAR-100, and ImageNet-1K. The one exception is ImageNet-200, where they rank second behind KNN. Using aggregation tends to increase the separation in this distance-based comparison, and it also removes the ImageNet-200 shortfall by taking the top rank there.

\paragraph{Near-OOD setting}
On Near-OOD, we again start with general post-hoc approaches. BBAS remains competitive, and aggregation match or exceed the general baselines on CIFAR-10, and on the larger datasets they usually sit just behind the best general method.
We then compare against distance-based methods. In this comparison the single-score BBAS variants are already in the same top cluster as RMDS/KNN on CIFAR-10, but on CIFAR-100 and ImageNet-200 they place behind them. On ImageNet-1k they place second behind RMDS. Importantly, BBAS with aggregation scores is the top performing method of the distance-based group.

We note that to date, no single method has consistently outperformed its competitors across every dataset and model architecture~\cite{Tajwar2021NoTS}. However, the BBAS method has shown superiority in comparison with the distance-based baseline.

\begin{table*}[!htbp]
  \centering
  \caption{Average performances on FarOOD benchmarks using ResNet architectures. Best performance is in \textbf{bold}, second best is \underline{underlined}. We indicate a $-$ when the result is not provided in the original paper of the method nor in the OpenOOD leader board.}
  \vspace{1mm}
  \label{tab:global_perf_farood}
  \resizebox{\textwidth}{!}{%
    \begin{tabular}{lcccccccc}
      \toprule
      \multirow{2}{*}{\textbf{Method}}              &
      \multicolumn{2}{c}{\textbf{CIFAR-10}}         &
      \multicolumn{2}{c}{\textbf{CIFAR-100}}        &
      \multicolumn{2}{c}{\textbf{ImageNet-200}}     &
      \multicolumn{2}{c}{\textbf{ImageNet-1K}}                                                                  \\
      \cmidrule(lr){2-3} \cmidrule(lr){4-5} \cmidrule(lr){6-7} \cmidrule(lr){8-9}
                                                    & \textbf{FPR95 $\downarrow$} & \textbf{AUROC $\uparrow$} &
      \textbf{FPR95 $\downarrow$}                   & \textbf{AUROC $\uparrow$}   &
      \textbf{FPR95 $\downarrow$}                   & \textbf{AUROC $\uparrow$}   &
      \textbf{FPR95 $\downarrow$}                   & \textbf{AUROC $\uparrow$}                                 \\
      \midrule
      \multicolumn{9}{l}{\textbf{General approach methods}}                                                     \\
      MSP \cite{hendrycks2017baseline}              &
      $ 31.72 $                                     & $ 90.73 $                   &
      $ 58.70 $                                     & $ 77.76 $                   &
      $ 35.43 $                                     & $ 90.13 $                   &
      $ 51.45 $                                     & $ 85.23 $                                                 \\
      ODIN \cite{liang2018enhancing}                &
      $ 57.62 $                                     & $ 87.96 $                   &
      $ 58.86 $                                     & $ 79.28 $                   &
      $ 34.23 $                                     & $ 91.71 $                   &
      $ 43.96 $                                     & $ 89.47 $                                                 \\
      VIM~\cite{Wang_2022_CVPR}                     &
      $ 25.05 $                                     & $ 93.48 $                   &
      $ 50.74 $                                     & $ 81.70 $                   &
      $ 27.20 $                                     & $ 91.26 $                   &
      $ 24.67 $                                     & $ 92.68 $                                                 \\
      GEN~\cite{Liu2023GEN}                         &
      $ 34.73 $                                     & $ 91.35 $                   &
      $ 56.71 $                                     & $ 79.68 $                   &
      $ 32.10 $                                     & $ 91.36 $                   &
      $ 35.61 $                                     & $ 89.76 $                                                 \\
      SHE~\cite{zhang2023outofdistribution}         &
      $ 66.48 $                                     & $ 85.32 $                   &
      $ 64.12 $                                     & $ 76.92 $                   &
      $ 42.17 $                                     & $ 89.81 $                   &
      $ 41.45 $                                     & $ 90.92 $                                                 \\
      ASH \cite{djurisic2023extremely}              &
      $ 79.03 $                                     & $ 78.49 $                   &
      $ 59.20 $                                     & $ 80.58 $                   &
      $ 27.29 $                                     & $ 93.90 $                   &
      $ \underline{19.49} $                         & $ \underline{95.74} $                                     \\
      SCALE \cite{xu2024scaling}                    &
      $ - $                                         & $ 86.39 $                   &
      $ - $                                         & $ 81.42 $                   &
      $ - $                                         & $ 93.98 $                   &
      $ \mathbf{16.53} $                            & $ \mathbf{96.53} $                                        \\

      \multicolumn{9}{l}{\textbf{Distance-based methods}}                                                       \\
      MDS~\cite{lee2018simple}                      &
      $ 32.22 $                                     & $ 89.72 $                   &
      $ 72.26 $                                     & $ 69.39 $                   &
      $ 61.66 $                                     & $ 74.72 $                   &
      $ 62.92 $                                     & $ 74.25 $                                                 \\
      MDSEns~\cite{lee2018simple}                   &
      $ 61.47 $                                     & $ 73.90 $                   &
      $ 66.74 $                                     & $ 66.00 $                   &
      $ 80.96 $                                     & $ 69.27 $                   &
      $ 82.77 $                                     & $ 74.25 $                                                 \\
      RMDS~\cite{DBLP:journals/corr/abs-2106-09022} &
      $ 25.35 $                                     & $ 92.20 $                   &
      $ 52.81 $                                     & $ 82.92 $                   &
      $ 32.45 $                                     & $ 88.06 $                   &
      $ 40.91 $                                     & $ 86.38 $                                                 \\
      KNN~\cite{sun2022nearest}                     &
      $ 24.27 $                                     & $ 92.96 $                   &
      $ 53.65 $                                     & $ 82.40 $                   &
      $ 27.27 $                                     & $ 93.16 $                   &
      $ 34.13 $                                     & $ 90.18 $                                                 \\
      \midrule
      BBAS-EC (Ours)                                &
      $ 20.02 $                                     & $ 94.47 $                   &
      $ \mathbf{41.42} $                            & $ \underline{88.03} $       &
      $ 33.01 $                                     & $ 91.52 $                   &
      $ 34.81 $                                     & $ 91.83 $                                                 \\
      BBAS-ED (Ours)                                &
      $ \mathbf{17.62} $                            & $ \mathbf{95.15} $          &
      $ 42.45 $                                     & $ 87.26 $                   &
      $ 33.68 $                                     & $ 91.35 $                   &
      $ 36.77 $                                     & $ 91.20 $                                                 \\
      AggBBAS-EC (Ours)                             &
      $ \underline{20.09} $                         & $ \underline{94.50} $       &
      $ \underline{42.19} $                         & $ \mathbf{88.65} $          &
      $ \underline{22.50} $                         & $ \underline{94.83} $       &
      $ 22.39 $                                     & $ 95.11 $                                                 \\
      AggBBAS-ED (Ours)                             &
      $ 18.02$                                      & $ 95.08 $                   &
      $ 43.45 $                                     & $ 87.55 $                   &
      $ \mathbf{21.97} $                            & $ \mathbf{95.07} $          &
      $ 22.68 $                                     & $ 95.13 $                                                 \\
      \bottomrule
    \end{tabular}
  }
\end{table*}

\begin{table*}[htbp]
  \centering
  \caption{Average performances on NearOOD benchmarks using ResNet architectures. Best performance is in \textbf{bold}, second best is \underline{underlined}. We indicate a $-$ when the result is not provided in the original paper of the method nor in the OpenOOD leader board.}
  \vspace{1mm}
  \label{tab:global_perf_nearood}
  \resizebox{\textwidth}{!}{%
    \begin{tabular}{lcccccccc}
      \toprule
      \multirow{2}{*}{\textbf{Method}}              &
      \multicolumn{2}{c}{\textbf{CIFAR-10}}         &
      \multicolumn{2}{c}{\textbf{CIFAR-100}}        &
      \multicolumn{2}{c}{\textbf{ImageNet-200}}     &
      \multicolumn{2}{c}{\textbf{ImageNet-1K}}                                                                  \\
      \cmidrule(lr){2-3} \cmidrule(lr){4-5} \cmidrule(lr){6-7} \cmidrule(lr){8-9}
                                                    & \textbf{FPR95 $\downarrow$} & \textbf{AUROC $\uparrow$} &
      \textbf{FPR95 $\downarrow$}                   & \textbf{AUROC $\uparrow$}   &
      \textbf{FPR95 $\downarrow$}                   & \textbf{AUROC $\uparrow$}   &
      \textbf{FPR95 $\downarrow$}                   & \textbf{AUROC $\uparrow$}                                 \\
      \midrule
      \multicolumn{9}{l}{\textbf{General approach methods}}                                                     \\
      MSP~\cite{hendrycks2017baseline}              &
      $ 48.17 $                                     & $ 88.03 $                   &
      $ 54.80 $                                     & $ 80.27 $                   &
      $ 54.82 $                                     & $ 83.34 $                   &
      $ 65.68 $                                     & $ 76.02 $                                                 \\
      ODIN~\cite{liang2018enhancing}                &
      $ 76.19 $                                     & $ 82.87 $                   &
      $ 57.91 $                                     & $ 79.90 $                   &
      $ 66.76 $                                     & $ 80.27 $                   &
      $ 72.50 $                                     & $ 74.75 $                                                 \\
      VIM~\cite{Wang_2022_CVPR}                     &
      $ 44.84 $                                     & $ 88.68 $                   &
      $ 62.63 $                                     & $ 74.98 $                   &
      $ 59.19 $                                     & $ 78.68 $                   &
      $ 71.35 $                                     & $ 72.08 $                                                 \\
      GEN~\cite{Liu2023GEN}                         &
      $ 53.67 $                                     & $ 88.20 $                   &
      $ 54.42 $                                     & $ 81.31 $                   &
      $ 55.20 $                                     & $ 83.68 $                   &
      $ 65.32 $                                     & $ 76.85 $                                                 \\
      SHE~\cite{zhang2023outofdistribution}         &
      $ 79.65 $                                     & $ 81.54 $                   &
      $ 59.07 $                                     & $ 78.95 $                   &
      $ 66.80 $                                     & $ 80.18 $                   &
      $ 73.01 $                                     & $ 73.78 $                                                 \\
      ASH~\cite{djurisic2023extremely}              &
      $ 86.78 $                                     & $ 75.27 $                   &
      $ 65.71 $                                     & $ 78.20 $                   &
      $ 64.89 $                                     & $ 82.38 $                   &
      $ \underline{63.32} $                         & $ 78.17 $                                                 \\
      SCALE~\cite{xu2024scaling}                    &
      $ - $                                         & $ 84.84 $                   &
      $ - $                                         & $ \mathbf{80.99} $          &
      $ - $                                         & $ \mathbf{84.84} $          &
      $ \mathbf{59.76} $                            & $ \mathbf{81.36} $                                        \\

      \multicolumn{9}{l}{\textbf{Distance-based methods}}                                                       \\
      MDS~\cite{lee2018simple}                      &
      $ 49.90 $                                     & $ 84.20 $                   &
      $ 83.53 $                                     & $ 58.69 $                   &
      $ 79.11 $                                     & $ 61.93 $                   &
      $ 85.45 $                                     & $ 55.44 $                                                 \\
      MDSEns~\cite{lee2018simple}                   &
      $ 92.96 $                                     & $ 60.43 $                   &
      $ 95.88 $                                     & $ 46.31 $                   &
      $ 91.75 $                                     & $ 54.32 $                   &
      $ 93.52 $                                     & $ 49.67 $                                                 \\
      RMDS~\cite{DBLP:journals/corr/abs-2106-09022} &
      $ 38.89 $                                     & $ 89.80 $                   &
      $ \mathbf{55.46} $                            & $ 80.15 $                   &
      $ \mathbf{54.02} $                            & $ 82.57 $                   &
      $ 65.04 $                                     & $ 76.99 $                                                 \\
      KNN~\cite{sun2022nearest}                     &
      $ 34.01 $                                     & $ 90.64 $                   &
      $ 61.22 $                                     & $ 80.18 $                   &
      $ 60.18 $                                     & $ 81.57 $                   &
      $ 70.87 $                                     & $ 71.10 $                                                 \\
      \midrule
      BBAS-EC (Ours)                                &
      $ 34.20 $                                     & $ 90.67 $                   &
      $ 60.74 $                                     & $ 77.45 $                   &
      $ 64.14 $                                     & $ 79.91 $                   &
      $ 71.90 $                                     & $ 75.54 $                                                 \\
      BBAS-ED (Ours)                                &
      $ 34.11 $                                     & $ 90.70 $                   &
      $ 60.58 $                                     & $ 77.34 $                   &
      $ 63.64 $                                     & $ 79.93 $                   &
      $ 73.77 $                                     & $ 74.31 $                                                 \\
      AggBBAS-EC (Ours)                             &
      $ \mathbf{34.00} $                            & $ \underline{90.74} $       &
      $ 58.23 $                                     & $ 79.96 $                   &
      $ 58.43 $                                     & $ 83.58 $                   &
      $ 66.17 $                                     & $ \underline{78.60} $                                     \\
      AggBBAS-ED (Ours)                             &
      $ \underline{34.03} $                         & $ \mathbf{90.76} $          &
      $ \underline{57.26} $                         & $ \underline{80.26} $       &
      $ \underline{57.46} $                         & $ \underline{84.09} $       &
      $ 68.03 $                                     & $ 77.90 $                                                 \\
      \bottomrule
    \end{tabular}
  }
\end{table*}

\subsection{Extension to transformer architectures}

Although the paper is mainly centered on convolutional architectures, the method itself is not restricted on these types of architectures. To apply BBAS to any architecture, we need only to decide on the clustering variables and the monitoring variables. To show this, we evaluated BBAS on transformer architectures (ViT-B/16~\citep{dosovitskiy2021imageworth16x16words} and Swin-T~\citep{liu2021swintransformerhierarchicalvision}) using the same previously specified hyper-parameters and choosing the output of their last encoder block as both monitoring and clustering variables. We follow the OpenOOD evaluation protocol exactly, using the official torchvision checkpoints~\citep{10.1145/1873951.1874254}.
The results of SCALE~\cite{xu2024scaling} are not available for transformer architectures. Extended results are provided in the supplementary material.

In the transformer setting, our BBAS variants remain consistently strong, with aggregation providing the most reliable gains. In particular, AggBBAS-EC is the most stable choice and delivers the strongest overall behavior in the Near-OOD regime across both ViT and Swin, while AggBBAS-ED is the most competitive BBAS variant for Far-OOD.

When compared to broader post-hoc baselines, the trends are also clear: on Far-OOD, stronger general-purpose methods (especially VIM) remain very competitive, and distance-based approaches form a strong upper tier. Our aggregated BBAS scores fall into this top group and are close to the best-performing baselines.
On Near-OOD, where many standard post-hoc methods degrade more noticeably, our single-score BBAS variants are already top-performing relative to the baseline methods and our aggregated variants separate more cleanly.

Our method demonstrates strong performance on these architectures, placing it among the best methods for the FarOOD tests and by far the top-performing method for NearOOD detection on transformers.
\begin{table*}
  \centering
  \caption{Results for Transformer tests on ImageNet-1k (AUROC). Best performance is in \textbf{bold}, second best is \underline{underlined}.}
  \vspace{1mm}
  \label{tab:transformers}
  \scriptsize
  \begin{tabular}{lcccc}
    \toprule
    \multirow{2}{*}{\textbf{Method}}              &
    \multicolumn{2}{c}{\textbf{FarOOD}}           &
    \multicolumn{2}{c}{\textbf{NearOOD}}                                                                                                          \\
    \cmidrule(lr){2-3} \cmidrule(lr){4-5}
                                                  & \textbf{ViT-B/16}     & \textbf{Swin-T}       & \textbf{ViT-B/16}     & \textbf{Swin-T}       \\
    \midrule
    \multicolumn{5}{l}{\textbf{General approach methods}}                                                                                         \\
    MSP~\cite{hendrycks2017baseline}              & $ 86.04 $             & $ 86.30 $             & $ 73.52 $             & $ 76.75 $             \\
    ODIN~\cite{liang2018enhancing}                & $ - $                 & $ 73.29 $             & $ - $                 & $ 68.15 $             \\
    VIM~\cite{Wang_2022_CVPR}                     & $ \mathbf{92.84} $    & $ \mathbf{93.20} $    & $ 77.03 $             & $ 75.24 $             \\
    SHE~\cite{zhang2023outofdistribution}         & $ 92.42 $             & $ 89.84 $             & $ 76.11 $             & $ 76.59 $             \\
    GEN~\cite{Liu2023GEN}                         & $ 91.35 $             & $ 90.98 $             & $ 76.30 $             & $ 78.97 $             \\
    ASH~\cite{djurisic2023extremely}              & $ 51.56 $             & $ 44.64 $             & $ 53.21 $             & $ 46.47 $             \\

    \multicolumn{5}{l}{\textbf{Distance-based methods}}                                                                                           \\
    MDS~\cite{lee2018simple}                      & $ \underline{92.60} $ & $ 91.49 $             & $ 79.04 $             & $ 75.18 $             \\
    RMDS~\cite{DBLP:journals/corr/abs-2106-09022} & $ \underline{92.60} $ & $ 92.34 $             & $ 80.09 $             & $ 78.34 $             \\
    KNN~\cite{sun2022nearest}                     & $ 90.81 $             & $ 89.37 $             & $ 74.11 $             & $ 71.62 $             \\
    \midrule
    \multicolumn{5}{l}{\textbf{Ours (distance-based)}}                                                                                            \\
    BBAS-EC (Ours)                                & $ 90.46 $             & $ 90.76 $             & $ \underline{81.97} $ & $ 81.31 $             \\
    BBAS-ED (Ours)                                & $ 91.20 $             & $ 91.29 $             & $ 81.75 $             & $ 80.97 $             \\
    AggBBAS-EC (Ours)                             & $ 92.02 $             & $ 92.98 $             & $ \mathbf{82.33} $    & $ \mathbf{83.02} $    \\
    AggBBAS-ED (Ours)                             & $ 92.14 $             & $ \underline{93.03} $ & $ 81.20 $             & $ \underline{82.33} $ \\
    \bottomrule
  \end{tabular}
  \normalsize
\end{table*}

\section{Discussion}
\label{sec:discussion}

Our method introduces a limited number of tunable parameters. As demonstrated in the ablation study in Section~\ref{sec:ablation}, these parameters have a minimal impact on the overall performance. Even suboptimal choices still yield relevant results in the experiments. The stability of BBAS with respect to its parameters allowed us to forgo the validation data provided by the OpenOOD framework for tuning purposes. The performances are close enough that there is no objective way to decide the best choice of parameters. The motivation behind the choice of the number of clusters as the square root of the number of examples and the complete linkage in clustering is to obtain clusters that are the most spherical, without having a large number of clusters. This is particularly relevant for CIFAR10 as shown on Figure 3.

Nevertheless, BBAS is built upon several key design choices, primarily the selection of monitoring variables, clustering features, the clustering method, and the OOD score. Other choices for the aforementioned design decisions still yield relevant performance, suggesting that the method is robust to variations in its configuration. Moreover, exploring alternative strategies, such as adaptive weighting schemes for monitoring variables or more sophisticated clustering techniques, may further enhance detection capabilities and improve the overall stability of the approach.

Finally, while this work and associated experiments focused on classification tasks, our method can be applied to other deep learning problems, such as detecting OOD samples in regression tasks. In the case of regression, the main adaptation consists in considering only one set of data and not extracting each class individually.

\section{Conclusion}

This paper introduces BBAS, a novel deep learning framework for out-of-distribution detection that leverages monitoring variables derived from hidden features of the neural network. By integrating a clustering approach, BBAS refines bounding box coverage to enhance detection. Additionally, we propose heuristics for computing an OOD score that quantifies the risk of a new input being out-of-distribution.

BBAS has been successfully applied to multiple convolutional neural network architectures trained for image classification. We conducted experiments on various OOD detection benchmarks and demonstrated that our method successfully distinguishes OOD samples from in-distribution data, particularly in scenarios involving covariate shift. The performances are competitive against the leading state-of-the-art post-hoc methods, achieving the best performances on CIFAR datasets, and consistent satisfying results on ImageNet datasets.

Future work could explore refining the method through another choice of monitoring variables, clustering features, and clustering algorithms. There are also other variations that can be explored in the Anomaly Score definition, such as adding a weighting scheme to each bounding box interval that takes into account the importance of the channel to the prediction.
Furthermore, integrating this approach into unsupervised learning frameworks, including anomaly detection and active learning~\citep{ren2021surveydeepactivelearning}, could further increase its applicability.
This work can pave the way to further research into the bounding box framework in OOD detection and other adjacent tasks, as understanding the bounding box abstraction could be a valuable paradigm to more reliable and safe deep learning.

\clearpage
%% For citations use: 
%%       \cite{<label>} ==> [1]

%%

%% If you have bib database file and want bibtex to generate the
%% bibitems, please use
%%
\small
\bibliographystyle{elsarticle-harv}
\bibliography{elsarticle-template-num.bbl}
\normalsize
%% else use the following coding to input the bibitems directly in the
%% TeX file.

%% Refer following link for more details about bibliography and citations.
%% https://en.wikibooks.org/wiki/LaTeX/Bibliography_Management

% \begin{thebibliography}{00}

% %% For numbered reference style
% %% \bibitem{label}
% %% Text of bibliographic item

% \bibitem{lamport94}
%   Leslie Lamport,
%   \textit{\LaTeX: a document preparation system},
%   Addison Wesley, Massachusetts,
%   2nd edition,
%   1994.

% \end{thebibliography}

\clearpage
%% The Appendices part is started with the command \appendix;
%% appendix sections are then done as normal sections
\appendix

\section{Runtime costs}
\label{app:runtime_costs}

We report in Table~\ref{tab:comput-cost} the runtime costs of the main components of the proposed pipeline on CIFAR-100.

\begin{table}[!h]
  \centering
  \begin{tabular}{lc}
    \toprule
    \textbf{Component}                                 & \textbf{Time (in s)} \\
    \midrule
    \textit{Bounding-box construction}                 & \textit{60}          \\
    \midrule
    Model inference                                    & 46                   \\
    Clustering                                         & 14                   \\
    Bounding-box calculation                           & $<0.1$               \\
    \midrule
    \textit{Test-time post-processing (whole dataset)} & \textit{4.8}         \\
    \midrule
    Model inference                                    & 4                    \\
    Single anomaly score calculation                   & 0.8                  \\
    \bottomrule
  \end{tabular}
  \caption{Runtime costs of the main components of the proposed pipeline on CIFAR-100, that are \textit{Bounding-box construction} and \textit{Test-time post-processing on the whole dataset}. Breakdown times into sub-components is reported under each.}
  \label{tab:comput-cost}
\end{table}

% \paragraph{Bounding-box construction}
% The total runtime is approximately 60\,s, with the following breakdown:
% \begin{itemize}
%     \item Model inference: 46\,s
%     \item Clustering: 14\,s
%     \item Bounding-box calculation: $<0.1$\,s
% \end{itemize}

% \paragraph{Test-time post-processing on the whole dataset}
% The total runtime is 4.8\,s, with the following breakdown:
% \begin{itemize}
%     \item Model inference: 4\,s
%     \item Single anomaly score calculation: 0.8\,s
% \end{itemize}

\section{On the geometric structure of neural network bounding box}
\label{app:bb_geometry}

\paragraph{Notations} Feed-forward neural networks are a special type of parameterized functions that takes an input and produces an output:
\[
  \mathcal{N} : X \rightarrow  Y.
\]
Typically, we have
\begin{equation*}
  X = \mathbb{R}^{n_{\mathrm{in}}} \text{  and } Y = \mathbb{R}^{n_{\mathrm{out}}},
\end{equation*}
where $n_{\mathrm{in}}$ is the dimension of the input and $n_{\mathrm{out}}$ is the dimension of the output.

A  $k$-layers neural network $\mathcal{N}_k$ is defined as a composition of functions $f^{(\ell)}$ associated to different layers $\ell$ containing $n^{(\ell)}$ neurons:

\begin{align*}
  \mathcal{N}_k(x)         & =g \circ f^{(k)} \circ ... \circ f^{(2)} \circ f^{(1)}(x), \\
  f^{(\ell)}(z^{(\ell-1)}) & =\sigma(W^{(\ell)}z^{(\ell-1)}+b^{(\ell)}),                \\
  z^{(\ell)}               & = f^{(\ell)}(z^{(\ell-1)}),                                \\
  z^{(0)}                  & = x,
\end{align*}
where \(b^{(\ell)}\) is a bias column vector in $\mathbb{R}^{n^{(\ell)}}$, \(W^{(\ell)}\) is the $n^{(\ell)} \times n^{(\ell-1)}$ weight matrix, $z^{(\ell-1)}$ is the input of layer $\ell$. $\sigma$ is the activation function of the intermediate layers (in this work, it is the same for all of the intermediate layers), and $g$ is the output activation function, it can be different from the activation function of the intermediate layers.
Each function $f^{(\ell)}$ maps its inputs from $\mathbb{R}^{n^{(\ell-1)}}$  into $\mathbb{R}^{n^{(\ell)}}$:
\[
  f^{(\ell)} : \mathbb{R}^{n^{(\ell-1)}} \rightarrow \mathbb{R}^{n^{(\ell)}}.
\]

In this section, we will look at the behaviour of the $\relu$ activation function. To simplify, we adopt the same notation for $\sigma: \mathbb{R} \rightarrow \mathbb{R} $ and $\sigma: \mathbb{R}^{n} \rightarrow \mathbb{R}^{n} $ (element-wise function) such as for $ x \in \mathbb{R}^{n} $,
\[
  \sigma(x)  = (\sigma(x_i))_{i=1,...,n}.
\]
Each layer $f^{(\ell)}$ is composed of $n^{(\ell)}$ neurons, each neuron $i$ can be represented as function  $f^{(\ell)}_{i} : \mathbb{R}^{n^{(\ell-1)}} \rightarrow \mathbb{R} $  such as
\[
  f^{(\ell)}_{i}(z^{(\ell-1)})=\sigma(W^{(\ell)}_{i}z^{(\ell-1)}+b^{(\ell)}_{i}),
\]
where $W^{(\ell)}_{i}$ is the $i^{th}$ line of $W^{(\ell)}$.

The bounding box abstraction defines an operational domain through the hidden features of the neural network. The idea is that the hidden features are all continuous functions from the input space to the set of real numbers. So, defining upper and lower bounds on this set of functions to construct a bounding box implies defining a domain in the input space representing the intersection of a series of sub-level and super-level sets.

When an instance propagates through a neural network, each neuron takes the previous layer outcome and linearly transforms it before passing it through an activation function. We can consider this series of linear transformations as a mapping that we use to define a bounding box over some point set.

\begin{definition}(Neural network transformation)
  Let $\mathcal{N}$ be a neural network with fixed trainable parameters $\theta$, the transformation $\phi_\theta: \mathbb{R}^{n_{\mathrm{in}}} \rightarrow \mathbb{R}^{\#\mathrm{neurons}} $ is defined as:
  \begin{equation}
    \phi_\theta(x) = (\phi_{\theta,i}^{(\ell)}(x))_{i,\ell} = (W^{(\ell)}_{i}z^{(\ell-1)}(x)+b^{(\ell)}_{i})_{i,\ell}.
  \end{equation}
\end{definition}
\begin{definition}(Neural network bounding box)
  \label{def:BB_def_app}
  Let $\mathcal{N}$ be a neural network, a bounding box over a finite point set $A \subset \mathbb{R}^{n_{\mathrm{in}}}$ and for a fixed vector of trainable parameters $\theta$ is
  \begin{equation*}
    \mathcal{B}(A, \theta) := \prod_{\ell,i} \left[\min\limits_{x \in \;A}(\phi_{\theta,i}^{(\ell)}(x)),\; \max\limits_{x \in \;A}(\phi_{\theta,i}^{(\ell)}(x))\right].
  \end{equation*}
  The region of space delimited by $\mathcal{B}(A, \theta)$ is
  \begin{equation*}
    \mathcal{R}_{\mathcal{B}}(A, \theta) := \{x \in \mathbb{R}^{n_{\mathrm{in}}} \; | \phi_\theta(x) \; \in \mathcal{B}(A, \theta) \}.
  \end{equation*}
\end{definition}
\begin{definition}(Single layer bounding box)
  The bounding box over a layer $l$ of $\mathcal{N}$ over a finite point set $A \subset \mathbb{R}^{n_{\mathrm{in}}}$ and for a fixed vector of trainable parameters $\theta$ is
  \begin{equation*}
    \mathcal{B}^{(\ell)}(A, \theta) := \prod_{i} \left[\min\limits_{x \in \;A}(\phi_{\theta,i}^{(\ell)}(x)),\; \max\limits_{x \in \;A}(\phi_{\theta,i}^{(\ell)}(x))\right].
  \end{equation*}
  The region of space delimited by $\mathcal{B}^{(\ell)}(A, \theta)$ is
  \begin{equation*}
    \mathcal{R}^{(\ell)}_{\mathcal{B}}(A, \theta) := \{x \in \mathbb{R}^{n_{\mathrm{in}}} \; | \phi^{(\ell)}_\theta(x) \; \in \mathcal{B}^{(\ell)}(A, \theta) \}.
  \end{equation*}
\end{definition}

The shape of $\mathcal{R}_{\mathcal{B}}(A, \theta)$ depends on the architecture of the neural network and on the parameters $\theta$. In this appendix, we will analyse the resulting shape of bounding box in the input space.

\subsection{Bounding box over the first layer}

The shape of the region corresponding to the bounding box over the first layer can be explicitly defined.
\begin{lemma}
  \label{firstlayer}
  For $\mathcal{N}$ a neural network with $\relu$ activation on the first layer and $A \subset \mathbb{R}^{n_{\mathrm{in}}}$ a finite point set, we have:
  \begin{equation*}
    \mathcal{R}^{(1)}_{\mathcal{B}}(A, \theta) \text{ is a polytope} \Longleftrightarrow  \rank(W^{(1)})=n_{\mathrm{in}}.
  \end{equation*}
\end{lemma}

We note that the stated condition of full-rank is almost always true since the number of neurons of the first layer is often significantly higher than the input dimension.

\begin{proof}

  By definition
  \begin{equation}
    \label{eq:firstlayer-box}
    \mathcal R^{(1)}_{\mathcal B}(A,\theta)
    =\{x\in \mathbb{R}^{n_{\mathrm{in}}} \mid  \min\limits_{x \in \;A}(\phi_{\theta,i}^{(1)}(x))-b_i^{(1)}\le W_i^{(1)}x\le \max\limits_{x \in \;A}(\phi_{\theta,i}^{(1)}(x))-b_i^{(1)}\}.
  \end{equation}
  The right–hand side is an intersection of \(2n^{(1)}\) closed half-spaces, hence always a
  convex polyhedron. It is a \emph{polytope} precisely when it is bounded. We demonstrate below that this region is indeed bounded if and only if \(\rank(W^{(1)}) = n_{\mathrm{\mathrm{in}}}\).

    {\(\Longrightarrow)\)} Assume that \(\rank(W^{(1)})<n_{\mathrm{\mathrm{in}}}\).
  Then there is a non-zero vector \(v\in\ker W^{(1)}\).
  Pick any \(x_0\in\mathcal R^{(1)}_{\mathcal B}(A,\theta)\) and note that
  for every \(t\in\mathbb{R}\),
  \[
    \phi_{\theta}^{(1)}(x_0+tv)=W^{(1)}(x_0+tv)+b^{(1)}
    =W^{(1)}x_0+b^{(1)}\in\mathcal B .
  \]
  Thus \(x_0+tv\in \mathcal R^{(1)}_{\mathcal B}(A,\theta)\) for all \(t\in\mathbb{R}\),
  so the region contains an entire affine line and is unbounded.
  Consequently, if \(\mathcal R^{(1)}_{\mathcal B}(A,\theta)\) is a polytope,
  we must have \(\rank(W^{(1)})=n_{\mathrm{\mathrm{in}}}\).

    {\(\Longleftarrow)\)} Conversely, suppose \(\rank(W^{(1)})=n_{\mathrm{in}}\),
  so \(W^{(1)}\) is injective.
  Let \(\sigma_{\min}>0\) be its smallest singular value. Then:
  \[
    \|W^{(1)}x\|_2\;\ge\;\sigma_{\min}\|x\|_2\qquad\forall\,x\in\mathbb{R}^{n_{\mathrm{in}}}.
  \]
  Let:
  \[
    M:=\max_{i=1,\ldots,n^{(1)}}\max\bigl\{\,|\min\limits_{x \in \;A}(\phi_{\theta,i}^{(1)}(x))-b^{(1)}_i|,
    \,|\max\limits_{x \in \;A}(\phi_{\theta,i}^{(1)}(x))-b^{(1)}_i|\,\bigr\},
  \]
  so that for every \(x\) satisfying \eqref{eq:firstlayer-box}, \(\|W^{(1)}x\|_2\le\sqrt{n^{(1)}}\,M\). Note that $M$ is correctly defined because $A$ is a finite set.
  Combining with the inequality above gives
  \[
    \|x\|_2\;\le\;\frac{\sqrt{n^{(1)}}\,M}{\sigma_{\min}},
  \]
  a uniform bound for all \(x\in\mathcal R^{(1)}_{\mathcal B}(A,\theta)\).
  Hence the set is bounded, and being an intersection of finitely
  many half-spaces, it is indeed a polytope.
\end{proof}

\subsection{Activation Patterns and Bounding Boxes}
\label{sec:act_patterns}

A neural network that uses a piece-wise linear activation function like $\relu$ corresponds actually to a piece-wise linear function. Therefore, for a local region of space where all the activation functions are linear, the entire neural network will be subsequently locally linear. For a neural network with a $\relu$ activation function, each neuron either “turns on” or “turns off” depending on whether the preactivation output is positive or non-positive. Consequently, a combinatorial pattern of active units can characterize the entire network, providing a unique “signature” for a specific input or class of inputs, referred to as the \textit{activation pattern}.

\begin{definition}
  (Activation Patterns/Regions)
  Let $\mathcal{N}$ be a neural network with $\relu$ activation and an identity output activation. An activation pattern for $\mathcal{N}$ is the assignment of a bit to each neuron that characterizes its state (active or inactive):
  \begin{equation*}\label{eq:activ-pattern}
    a := (a^{(\ell)}_i) \in \mathcal{A}:=\{0, 1\}^{\#\mathrm{neurons}}.
  \end{equation*}
  For fixed trainable parameters $\theta$, there is a mapping $\Phi_{\theta}: \mathbb{R}^{n_{\mathrm{in}}}  \rightarrow  \mathcal{A}$, such as
  \begin{equation*}
    \Phi_{\theta}(x) := a
    \quad \text{with} \quad a^{(\ell)}_i=
    \left\{
    \begin{array}{rcl}
      0 & \quad \text{if} \quad z_i^{(\ell-1)} \leq 0 \\ 1 & \quad \text{if} \quad z_i^{(\ell-1)} > 0 .
    \end{array}
    \right.
  \end{equation*}
  The activation region corresponding to $a$ and $\theta$ is
  \begin{equation*}
    \mathcal{R}(a, \theta) := \Phi_\theta^{-1}(a) = \{x \in \mathbb{R}^{n_{\mathrm{in}}} | \Phi_{\theta}(x)=a \}.
  \end{equation*}
\end{definition}

Understanding the activation regions within a neural network is essential for understanding the geometric structure and overall behavior of the function it represents. Specifically, these activation regions determine the shape and segmentation of the neural network's functional domain across its multiple layers.

The activation region is a convex polyhedron resulting from the intersection of half-spaces defined by the linear inequalities representing the neurons' activation states.

\begin{lemma}(Convexity of activation regions)
  Let $\mathcal{N}$ be a $\relu$ neural network. Then every activation region $\mathcal{R}(a, \;\theta)$, with an activation pattern $a$ and a vector of trainable parameters $\theta$, is a convex polyhedron.
\end{lemma}

The proof follows directly from the definition of a convex polyhedron.

We can extend lemma~\ref{firstlayer} to all layers for points inside the same activation region using the fact that the neural network is linear on said activation region:

\begin{lemma}
  \label{alllayers}
  For $\mathcal{N}$ be a neural network and $A \subset \mathbb{R}^{n_{\mathrm{in}}}$ a finite point set such that there is an activation pattern $a$ with non-empty activation region such that $A \subset \mathcal{R}(a, \theta)$, the region $\mathcal{R}_{\mathcal{B}}(A, \theta)$ (see Definition~\ref{def:BB_def_app}) is convex and
  \begin{equation*}
    \mathcal{R}_{\mathcal{B}}(A, \theta) \; \text{is a polytope} \Longleftrightarrow  \rank(W^{(1)})=n_{\mathrm{in}}.
  \end{equation*}
\end{lemma}

Consequently, a region enclosed by a bounding box, defined over a finite set of points and a neural network where the weight matrix of the first layer has full rank relative to the input space, forms a union of bounded polytopes. These polytopes correspond to activation regions determined by the possible activation patterns of points within the bounding box. For example, if the set of points exhibits only two adjacent activation patterns $a$, meaning they differ at most in one bit, the region enclosed by the bounding box will consist of exactly two bounded polytopes. Conversely, if there are three adjacent activation patterns, meaning they differ at most in two bits, the region will comprise at most four bounded polytopes, depending on whether the fourth possible activation pattern has a nonempty activation region.

\paragraph{Formally}
Let \(d\in\mathbb{N}^*\) the number of neurons in a neural network and let a bounding box
\[
  \mathcal{B}(A, \theta) \;=\; \bigl\{\,x\in\mathbb{R}^{d}\ \bigl|\
  \ell_i \,\le x_i \le\, u_i\ \text{for all } i=1,\dots,d\bigr\},
\]
be given, with lower and upper bounds \(l=(l_1,\dots,l_d)\) and
\(u=(u_1,\dots,u_d)\) satisfying \(l_i\le u_i\).
Define
\[
  P \;:=\; \bigl\{\,i\in\{1,\dots,d\}\,\bigl|\,
  l_i < 0 < u_i \bigr\},
  \qquad n_b \;:=\; |P|,
\]
so \(n_b\) counts the coordinates in which the box spans both negative
and positive values (i.e.\ “crosses the origin’’).

For every subset \(S\subseteq P\) let
\[
  B_S \;:=\;
  \bigl\{\,x\in \mathcal{B}(A, \theta) \bigl|\,
  x_i \ge 0 \text{ if } i\in S,\ \,
  x_i \le 0 \text{ if } i\in P\setminus S
  \bigr\}.
\]
Then each \(B_S\) lies entirely inside a single orthant (all coordinates
have a fixed sign pattern), and
\[
  \mathcal{B}(A, \theta)
  \;=\;
  \bigcup_{S\subseteq P} B_S,
  \qquad
  \text{with at most } 2^{n_b} \text{ (non-empty) boxes } B_S.
\]
Consequently:
\begin{multline*}
  \left[
    \mathcal{R}_{\mathcal{B}}(A, \theta)
    \;=\;
    \bigcup_{S\subseteq P} \mathcal{R}_{\mathcal{B_S}}
    \quad
    \text{with at most } 2^{n_b} \text{ (non-empty) polytopes } \mathcal{R}_{\mathcal{B_S}}
    \right]\\
  \;\Longleftrightarrow\;
  \left[ \operatorname{rank}\!\bigl(W^{(1)}\bigr)=n_{\mathrm{in}} \right].
\end{multline*}

Here, \(n_b\) denotes the number of possible activation patterns, not the number of patterns actually grouped into the cluster.

\subsection{Activation patterns and clustering}
\label{sec:act_patterns_clust}

Clustering should be understood through the geometry induced by the ReLU network. As shown in \ref{sec:act_patterns}, the network partitions the input space into activation regions on which the network is affine. When we later enclose a cluster in a bounding box, the box typically intersects several activation regions. In particular, if the box spans both sides of enough activation boundaries (captured by \(n_b\) in the previous section), then the box may contain up to \( 2^{n_b} \) (non-empty) polytopes. This leads to a simple principle that connects clustering to activation patterns without committing to any particular clustering features or metric: clustering procedures that produce approximately spherical (isotropic) clusters in activation space are the ones that minimize the volume covered by the abstraction.
If a cluster in activation space is elongated, tilted, or has “tentacles”, then covering it with per-coordinate minima and maxima introduces some amount of empty volume inside the box.

A convenient way to operationalize this principle is to cluster using the network activation patterns. Activation patterns provide a discrete, geometry-aware description of how the network processes inputs: two inputs are considered close if they trigger almost the same set of active ReLU units. In that setting, spherical clusters correspond to groups that lie in a small ball around a typical pattern, rather than being spread across many unrelated patterns. A natural metric for enforcing such ball-like structure is the Hamming distance.

\begin{definition}
  (Hamming distance)
  Let $x$ and $y$ be two binary vectors in $\{0,1\}^n$. The Hamming distance $d(x,y)$ between $x$ and $y$ is the number of elements (bits) that differ between both vectors:
  \begin{equation}\label{eq:hamming_dist}
    d(x,y) = \mathrm{Card}(\{i \in \llbracket 1, \ldots, n \rrbracket \mid x_i \neq y_i\}).
  \end{equation}
\end{definition}

We highlight that the Hamming distance can be used because activation patterns are binary vectors. This method aligns naturally with the geometric insights discussed previously, as the combinatorial nature of activation patterns within a bounding box is directly related to its coverage.

\paragraph{Activation patterns and network capacity} Activation patterns significantly reflect the capacity and expressivity of neural networks. In particular, the number of linear regions induced by networks has become an established metric for assessing expressivity, as networks with more regions can approximate increasingly complex functions. Although counting activation regions exactly is NP-hard in general~\cite{ijcai2022p492}, various theoretical methods provide bounds or estimations~\cite{montfar2014numberlinearregionsdeep, SerraThiag}. These analyses typically use combinatorial arguments involving hyperplane arrangements defined by \(\relu\) activations, where each neuron, through its \(\relu\) activation boundary, partitions the input space, with deeper networks exponentially multiplying these partitions via hierarchical layer stacking, effectively creating repeated subdivisions. \cite{telgarsky2015representationbenefitsdeepfeedforward} construction of a \(2^L\)-sawtooth function using a \(\relu\) network with \(O(L)\) neurons exemplifies this exponential growth.
Nevertheless, maximum theoretical capacity differs significantly from practical network usage. \citet{raghu17a} confirm this by measuring network expressivity by tracking the trajectory complexity, defined as the number of activation region boundary crossings along random input lines, during training. Despite initial increases, they found that the number of region crossings tends to plateau or grow sub-exponentially as training proceeds. In other words, the network does create more linear pieces to fit the data, though not in the combinatorially explosive way one might expect if it were using all available neurons independently. Furthermore, \citet{pmlr-v97-hanin19a, Hanin} demonstrate that typical neural networks, both at random initialization and following conventional training, do not attain worst-case activation complexity. Instead, their observed activation patterns are substantially simpler than those worst-case constructions. Specifically, they observe that the number of linear regions grows approximately linearly with the total number of neurons rather than exponentially, which explains, in part, why over-parameterized networks can generalize well. Also, they show that although the network function space is vast, stochastic gradient descent may only explore a small, structured subset biased toward “simpler” functions (in terms of activation regions). For instance, slight Gaussian perturbations to parameters in \cite{telgarsky2015representationbenefitsdeepfeedforward} sawtooth network drastically simplify its activation partitions (many teeth disappear). Thus, gradient descent implicitly favors simpler, broader partitions covering data groups rather than complex, arbitrary tilings.
These results show that neural networks possess immense combinatorial expressivity through activation patterns. However, training regimes and implicit regularization significantly restrict this capacity in practice as neural networks tend toward simplified, structured activation partitions rather than utilizing their full combinatorial complexity. Consequently, clustering samples will not attain the full \( 2^{n_b} \) (non-empty) polytopes.

\paragraph{Activation patterns and data structure} Activation patterns tend to align with the intrinsic structure of data. \cite{novak2018sensitivitygeneralizationneuralnetworks} analyze neural networks on image classification tasks and introduce sensitivity measures related to the geometry of activation regions. Their findings indicate that networks exhibiting better generalization capabilities possess larger, smoother activation regions around data points, enhancing robustness against input perturbations within the data manifold. Conversely, networks trained on randomly labeled data, which enforce memorization, produce highly irregular decision boundaries with numerous small activation regions around individual data points. This results in significant sensitivity, where small perturbations easily shift the input across region boundaries and alter outputs.
Furthermore, activation patterns contain the task-relevant invariances encoded by neural networks. \(\relu\) networks encode invariances through a mechanism known as space folding, in which spatially distant inputs map to similar activation patterns. \cite{lewandowski2025spacefoldsreluneural} characterize this phenomenon by analyzing linear input paths and corresponding activation trajectories represented within a Hamming cube. Ideally, a linear path in input space would correspond to a convex path in activation space. However, the folding phenomenon disrupts this linearity, yielding non-convex trajectories. Notably, \cite{lewandowski2025spacefoldsreluneural} observe cases where the Hamming distance to the initial input point decreased despite increases in Euclidean distance in input space. They define a measure based on range metrics reflecting deviations in Hamming distance from linear interpolation to quantify space folding. Their analysis of synthetic benchmarks (e.g., CantorNet) and real datasets (e.g., MNIST) demonstrate that folding increases with network depth and is more pronounced in networks with better generalization. Thus, networks with higher generalization performance exhibit greater degrees of space folding, indicating that learned invariances manifest structurally as folds. \cite{lewandowski2025spacefoldsreluneural} conclude that network generalization capabilities are closely tied to their capacity for space folding and invariance learning.

We add that the Hamming distance between activation patterns also corresponds to the difference in linear transformation applied to the concerned activation regions. Notably, a small Hamming distance between the activation patterns of two inputs implies subtle, low-dimensional differences in the learned representation, even if the original inputs are not close in the input space. Accordingly, clustering the training data using the Hamming distance between activation patterns effectively uses the network-learned representation of the input space.

Especially, we can state the following lemma:

\begin{lemma}\label{prop:Hamming}
  A Hamming distance of one between two activation patterns implies a rank-one difference in the linear transformation matrix that corresponds to these activation patterns.
\end{lemma}

\begin{proof}
  Let us fix an activation pattern $\alpha$, see Definition~\ref{eq:activ-pattern}. Then the network's mapping from input $x$ to output $y$ can be written as a single affine linear transformation:
  \[ y = Ax+b
  \]
  where $A$ and $b$ depend on the pattern  $\alpha$ of which neurons are active.
  For a network with $L$ layers, each layer $\ell$ transforms its input $h^{(\ell -1)}$ into a pre-activation
  \[
    z^{(\ell )} = W^{(\ell )} h^{(\ell -1)} + b^{(\ell )}.
  \]

  Once an activation pattern  $\alpha$ is fixed, the \(\relu\) operation for each layer reduces to a linear selection. This can be represented by a diagonal matrix $D^{(\ell )}$ whose diagonal entries are either 1 (active neuron) or 0 (inactive neuron). Thus, on a given region:
  \[
    h^{(\ell )} := \relu(z^{(\ell )}) = D^{(\ell )} \left(W^{(\ell )} h^{(\ell -1)} + b^{(\ell )}\right),
  \]
  where $D^{(\ell )}$ is determined by the sign pattern of $z^{(\ell )}$.

  Chaining these together from the input to the output, the linear transformation matrix is:
  \[
    A = \prod_{k=0}^{L-1}D^{(L-k)} W^{(L-k)}.
  \]

  Now, consider that we move to a region of the input space where the activation pattern differs at \textit{exactly one neuron}. That means for some layer $i$ and some neuron $j$ in that layer, the corresponding diagonal entry in $D^{(i)}$ flips from 1 to 0, or from 0 to 1, and we denote by $D'^{(i)}$ the corresponding diagonal matrix. The linear transformation matrix $A$ becomes:

  \[
    A' = A_1 D'^{(i)} W^{(i)} A_2,
  \]

  where:
  \[
    A_1=\prod_{k=0}^{L-i+1}D^{(L-k)} W^{(L-k)}, \text{ and }
    A_2=\prod_{k=L-i-1}^{L-1}D^{(L-k)} W^{(L-k)}.
  \]

  Let us say that originally $D^{(i)}$ had a 0 at position $j$, and now $D'^{(i)}$ has a 1 there. This means previously the $j$-th neuron's output of layer $i$ was set to zero, and now it is allowed through. Algebraically, $D'^{(i)} - D^{(i)}$ is the diagonal matrix with a one in the $(j, j)$ position and 0 elsewhere.
  Substituting back into the full product yields:
  \[
    A' - A = A_1 \left( D'^{(i)} - D^{(i)} \right) W^{(i)} A_2.
  \]

  Since $D'^{(i)} - D^{(i)}$ is zero everywhere except one diagonal entry, this difference has a rank-1 effect on the overall matrix $A$. In other words, adding or removing one active neuron at layer $i$ corresponds to adding or removing a single direction of linear influence.
\end{proof}

% \section{Additional t-SNE Visualizations}
% \label{sec:tsne-appendix}

% This appendix presents additional t-SNE embeddings for alternative feature representations extracted from a ResNet trained on CIFAR-10.
% Specifically, Figure~\ref{fig:tsne-appendix-all} visualize per-channel minimum activations, per-channel activation fractions, and penultimate-layer features.
% Across all representations, the embeddings exhibit consistent class-wise clustering with low intra-class variance and clear inter-class separation, indicating that the observed class-dependent structure is robust to the choice of feature representation.

\section{Detailed Experiments Results}
\label{sec:detailedres}

This section presents the detailed results of the comparison with the state-of-the-art methods using the OpenOOD benchmark. For CIFAR-10, we report both FarOOD and NearOOD results on ResNet18 in Tables~\ref{tab:cifar10_benchmark} and \ref{tab:cifar10_benchmark/nearOOD}. For CIFAR-100, the corresponding FarOOD and NearOOD results on ResNet18 are provided in Tables~\ref{tab:cifar100_benchmark} and \ref{tab:cifar100_benchmark/nearOOD}. For ImageNet-200, FarOOD and NearOOD results on ResNet18 are reported in Tables~\ref{tab:imagenet200_benchmark} and \ref{tab:imagenet200_benchmark/nearOOD}. Finally, for ImageNet-1k, we present results on multiple architectures: ResNet50 (FarOOD in Table~\ref{tab:imagenet1k_benchmark} and NearOOD in Table~\ref{tab:imagenet1k_benchmark/nearOOD}), ViT-B/16 (FarOOD in Table~\ref{tab:imagenet1k_benchmark/ViT} and NearOOD in Table~\ref{tab:imagenet1k_benchmark/ViT/nearOOD}), and Swin-T (FarOOD in Table~\ref{tab:imagenet1k_benchmark/swin} and NearOOD in Table~\ref{tab:imagenet1k_benchmark/swin/nearOOD}). The reported results come from the associated papers and from the OpenOOD leaderboard; when a metric is not available in the corresponding reference, we denote it with the symbol ``$-$'' in the tables. Across all tables, the best performance is highlighted in \textbf{bold}, while the second best is \underline{underlined}.

We provide here the list of the compared state-of-the-art methods and their associated paper: OpenMax \cite{DBLP:journals/corr/BendaleB15}, MSP \cite{HendrycksG17}, TempScale \cite{GuoPSW17}, ODIN \cite{liang2018enhancing}, MDS and MDSEns \cite{lee2018simple}, RMDS \cite{DBLP:journals/corr/abs-2106-09022}, Gram \cite{pmlr-v119-sastry20a}, EBO \cite{liu2020energy}, OpenGAN \cite{Kong_2021_ICCV}, GradNorm \cite{NEURIPS2021_063e26c6}, ReAct \cite{NEURIPS2021_01894d6f}, MLS and KLM \cite{pmlr-v162-hendrycks22a}, VIM \cite{Wang_2022_CVPR}, KNN \cite{sun2022nearest}, DICE \cite{sun2022dice}, RankFeat \cite{song2022rankfeat}, ASH \cite{djurisic2023extremely}, SHE \cite{zhang2023outofdistribution}, GEN \cite{Liu2023GEN}, SCALE \cite{xu2024scaling}.

\subsection{Full CIFAR-10 results:}

\FloatBarrier
\begin{table*}[!h]
  \centering
  \caption{CIFAR-10 FarOOD Benchmark results on the ResNet18 architecture.}
  \label{tab:cifar10_benchmark}
  \resizebox{\textwidth}{!}{%
\begin{tabular}{lcccccccccc}
      \toprule
      \multirow{2}{*}{\textbf{Method}}       &
      \multicolumn{2}{c}{\textbf{MNIST}}     &
      \multicolumn{2}{c}{\textbf{SVHN}}      &
      \multicolumn{2}{c}{\textbf{Textures}}  &
      \multicolumn{2}{c}{\textbf{Places365}} &
      \multicolumn{2}{c}{\textbf{Average}}                                                                                                                                                                                                                                                                                                                                                                                                                                                                                                   \\
      \cmidrule(lr){2-3} \cmidrule(lr){4-5} \cmidrule(lr){6-7} \cmidrule(lr){8-9} \cmidrule(lr){10-11}
                                             & \textbf{FPR95 $\downarrow$}                   & \textbf{AUROC $\uparrow$}                     &
      \textbf{FPR95 $\downarrow$}            & \textbf{AUROC $\uparrow$}                     &
      \textbf{FPR95 $\downarrow$}            & \textbf{AUROC $\uparrow$}                     &
      \textbf{FPR95 $\downarrow$}            & \textbf{AUROC $\uparrow$}                     &
      \textbf{FPR95 $\downarrow$}            & \textbf{AUROC $\uparrow$}                                                                                                                                                                                                                                                                                                                                                                                                                                                                     \\
\bottomrule
    \end{tabular}
  }
\end{table*}

\begin{table}[ht]
  \centering
  \caption{CIFAR-10 NearOOD Benchmark results on the ResNet18 architecture}
  \label{tab:cifar10_benchmark/nearOOD}
  \resizebox{\textwidth}{!}{%
    \begin{tabular}{lcccccccccc}
      \toprule
      \multirow{2}{*}{\textbf{Method}}       &
      \multicolumn{2}{c}{\textbf{CIFAR-100}} &
      \multicolumn{2}{c}{\textbf{TIN}}       &
      \multicolumn{2}{c}{\textbf{Average}}                                                                                                                                                                                                                                                                                                   \\
      \cmidrule(lr){2-3} \cmidrule(lr){4-5} \cmidrule(lr){6-7}
                                             & \textbf{FPR95 $\downarrow$}                   & \textbf{AUROC $\uparrow$}                     &
      \textbf{FPR95 $\downarrow$}            & \textbf{AUROC $\uparrow$}                     &
      \textbf{FPR95 $\downarrow$}            & \textbf{AUROC $\uparrow$}                                                                                                                                                                                                                                                                     \\
      \midrule
      OpenMax                                & $ 48.06 \pm {\scriptstyle 3.25} $             & $ 86.91 \pm {\scriptstyle 0.31} $             & $ 39.18 \pm {\scriptstyle 1.44} $             & $ 88.32 \pm {\scriptstyle 0.28} $             & $ 43.62 \pm {\scriptstyle 2.27} $             & $ 87.62 \pm {\scriptstyle 0.29} $             \\
      MSP                                    & $ 53.08 \pm {\scriptstyle 4.86} $             & $ 87.19 \pm {\scriptstyle 0.33} $             & $ 43.27 \pm {\scriptstyle 3.00} $             & $ 88.87 \pm {\scriptstyle 0.19} $             & $ 48.17 \pm {\scriptstyle 3.92} $             & $ 88.03 \pm {\scriptstyle 0.25} $             \\
\bottomrule
    \end{tabular}
  }
\end{table}

\FloatBarrier

\newpage

\subsection{Full CIFAR-100 results:}

\FloatBarrier

\begin{table}[!h]
  \centering
  \caption{CIFAR-100 FarOOD Benchmark results on the ResNet18 architecture}
  \label{tab:cifar100_benchmark}
  \resizebox{\textwidth}{!}{%
    \begin{tabular}{lcccccccccc}
      \toprule
      \multirow{2}{*}{\textbf{Method}}       &
      \multicolumn{2}{c}{\textbf{MNIST}}     &
      \multicolumn{2}{c}{\textbf{SVHN}}      &
      \multicolumn{2}{c}{\textbf{Textures}}  &
      \multicolumn{2}{c}{\textbf{Places365}} &
      \multicolumn{2}{c}{\textbf{Average}}                                                                                                                                                                                                                                                                                                                                                                                                                                                                                                    \\
      \cmidrule(lr){2-3} \cmidrule(lr){4-5} \cmidrule(lr){6-7} \cmidrule(lr){8-9} \cmidrule(lr){10-11}
                                             & \textbf{FPR95 $\downarrow$}                    & \textbf{AUROC $\uparrow$}                     &
      \textbf{FPR95 $\downarrow$}            & \textbf{AUROC $\uparrow$}                      &
      \textbf{FPR95 $\downarrow$}            & \textbf{AUROC $\uparrow$}                      &
      \textbf{FPR95 $\downarrow$}            & \textbf{AUROC $\uparrow$}                      &
      \textbf{FPR95 $\downarrow$}            & \textbf{AUROC $\uparrow$}                                                                                                                                                                                                                                                                                                                                                                                                                                                                      \\
\bottomrule
    \end{tabular}
  }
\end{table}

\begin{table}[ht]
  \centering
  \caption{CIFAR-100 NearOOD Benchmark results on the ResNet18 architecture}
  \label{tab:cifar100_benchmark/nearOOD}
  \resizebox{\textwidth}{!}{%
    \begin{tabular}{lcccccccccc}
      \toprule
      \multirow{2}{*}{\textbf{Method}}      &
      \multicolumn{2}{c}{\textbf{CIFAR-10}} &
      \multicolumn{2}{c}{\textbf{TIN}}      &
      \multicolumn{2}{c}{\textbf{Average}}                                                                                                                                                                                                                                                                                                  \\
      \cmidrule(lr){2-3} \cmidrule(lr){4-5} \cmidrule(lr){6-7}
                                            & \textbf{FPR95 $\downarrow$}                   & \textbf{AUROC $\uparrow$}                     &
      \textbf{FPR95 $\downarrow$}           & \textbf{AUROC $\uparrow$}                     &
      \textbf{FPR95 $\downarrow$}           & \textbf{AUROC $\uparrow$}                                                                                                                                                                                                                                                                     \\
      \midrule
      OpenMax                               & $ 60.17 \pm {\scriptstyle 0.97} $             & $ 74.38 \pm {\scriptstyle 0.37} $             & $ 52.99 \pm {\scriptstyle 0.51} $             & $ 78.44 \pm {\scriptstyle 0.14} $             & $ 56.58 \pm {\scriptstyle 0.73} $             & $ 76.41 \pm {\scriptstyle 0.25} $             \\
      MSP                                   & $ 58.91 \pm {\scriptstyle 0.93} $             & $ 78.47 \pm {\scriptstyle 0.07} $             & $ 50.70 \pm {\scriptstyle 0.34} $             & $ 82.07 \pm {\scriptstyle 0.17} $             & $ 54.80 \pm {\scriptstyle 0.33} $             & $ 80.27 \pm {\scriptstyle 0.11} $             \\
\bottomrule
    \end{tabular}
  }
\end{table}

\FloatBarrier

\newpage

\subsection{Full ImageNet-200 results:}

\FloatBarrier

\begin{table}[ht]
  \centering
  \caption{ImageNet-200 FarOOD Benchmark results on the ResNet18 architecture}
  \label{tab:imagenet200_benchmark}
  \resizebox{\textwidth}{!}{%
    \begin{tabular}{lcccccccccc}
      \toprule
      \multirow{2}{*}{\textbf{Method}}         &
      \multicolumn{2}{c}{\textbf{iNaturalist}} &
      \multicolumn{2}{c}{\textbf{Textures}}    &
      \multicolumn{2}{c}{\textbf{OpenImage-O}} &
      \multicolumn{2}{c}{\textbf{Average}}                                                                                                                                                                                                                                                                                                                                                                                                     \\
      \cmidrule(lr){2-3} \cmidrule(lr){4-5} \cmidrule(lr){6-7} \cmidrule(lr){8-9} \cmidrule(lr){10-11}
                                               & \textbf{FPR95 $\downarrow$}                   & \textbf{AUROC $\uparrow$}                     &
      \textbf{FPR95 $\downarrow$}              & \textbf{AUROC $\uparrow$}                     &
      \textbf{FPR95 $\downarrow$}              & \textbf{AUROC $\uparrow$}                     &
      \textbf{FPR95 $\downarrow$}              & \textbf{AUROC $\uparrow$}                                                                                                                                                                                                                                                                                                                                                                     \\
      \midrule
      OpenMax                                  & $ 24.53 \pm {\scriptstyle 0.96} $             & $ 92.32 \pm {\scriptstyle 0.32} $             & $ 36.80 \pm {\scriptstyle 0.55} $             & $ 90.21 \pm {\scriptstyle 0.07} $             & $ 38.03 \pm {\scriptstyle 0.49} $             & $ 88.07 \pm {\scriptstyle 0.14} $             & $ 33.12 \pm {\scriptstyle 0.66} $             & $ 90.20 \pm {\scriptstyle 0.17} $             \\
\bottomrule
    \end{tabular}
  }
\end{table}

\begin{table}[ht]
  \centering
  \caption{ImageNet-200 NearOOD Benchmark results on the ResNet18 architecture}
  \label{tab:imagenet200_benchmark/nearOOD}
  \resizebox{\textwidth}{!}{%
    \begin{tabular}{lcccccccccc}
      \toprule
      \multirow{2}{*}{\textbf{Method}}      &
      \multicolumn{2}{c}{\textbf{SSB-hard}} &
      \multicolumn{2}{c}{\textbf{NINCO}}    &
      \multicolumn{2}{c}{\textbf{Average}}                                                                                                                                                                                                                                                                                                  \\
      \cmidrule(lr){2-3} \cmidrule(lr){4-5} \cmidrule(lr){6-7}
                                            & \textbf{FPR95 $\downarrow$}                   & \textbf{AUROC $\uparrow$}                     &
      \textbf{FPR95 $\downarrow$}           & \textbf{AUROC $\uparrow$}                     &
      \textbf{FPR95 $\downarrow$}           & \textbf{AUROC $\uparrow$}                                                                                                                                                                                                                                                                     \\
      \midrule
      OpenMax                               & $ 72.37 \pm {\scriptstyle 0.11} $             & $ 77.53 \pm {\scriptstyle 0.08} $             & $ 54.59 \pm {\scriptstyle 0.54} $             & $ 83.01 \pm {\scriptstyle 0.17} $             & $ 63.48 \pm {\scriptstyle 0.25} $             & $ 80.27 \pm {\scriptstyle 0.10} $             \\
      MSP                                   & $ \mathbf{66.00} \pm {\scriptstyle 0.10} $    & $ 80.38 \pm {\scriptstyle 0.03} $             & $ 43.65 \pm {\scriptstyle 0.75} $             & $ 86.29 \pm {\scriptstyle 0.11} $             & $ 54.82 \pm {\scriptstyle 0.35} $             & $ 83.34 \pm {\scriptstyle 0.06} $             \\
\bottomrule
    \end{tabular}
  }
\end{table}

\FloatBarrier

\clearpage

\subsection{Full ImageNet-1k results:}

\FloatBarrier

\begin{table}[ht]
  \centering
  \caption{ImageNet-1k FarOOD Benchmark results on the ResNet50 architecture.}
  \label{tab:imagenet1k_benchmark}
  \resizebox{\textwidth}{!}{%
    \begin{tabular}{lcccccccccc}
      \toprule
      \multirow{2}{*}{\textbf{Method}}         &
      \multicolumn{2}{c}{\textbf{iNaturalist}} &
      \multicolumn{2}{c}{\textbf{Textures}}    &
      \multicolumn{2}{c}{\textbf{OpenImage-O}} &
      \multicolumn{2}{c}{\textbf{Average}}                                                                                                                                                                                                               \\
      \cmidrule(lr){2-3} \cmidrule(lr){4-5} \cmidrule(lr){6-7} \cmidrule(lr){8-9} \cmidrule(lr){10-11}
                                               & \textbf{FPR95 $\downarrow$} & \textbf{AUROC $\uparrow$} &
      \textbf{FPR95 $\downarrow$}              & \textbf{AUROC $\uparrow$}   &
      \textbf{FPR95 $\downarrow$}              & \textbf{AUROC $\uparrow$}   &
      \textbf{FPR95 $\downarrow$}              & \textbf{AUROC $\uparrow$}                                                                                                                                                                               \\
      \midrule
      OpenMax                                  & $ 25.29 $                   & $ 92.05 $                 & $ 40.26 $             & $ 88.10 $             & $ 37.39 $             & $ 87.62 $             & $ 34.31 $             & $ 89.26 $             \\
      MSP                                      & $ 43.34 $                   & $ 88.41 $                 & $ 60.87 $             & $ 82.43 $             & $ 50.13 $             & $ 84.86 $             & $ 51.45 $             & $ 85.23 $             \\
      TempScale                                & $ 37.63 $                   & $ 90.50 $                 & $ 56.90 $             & $ 84.95 $             & $ 45.40 $             & $ 87.22 $             & $ 46.64 $             & $ 87.56 $             \\
\bottomrule
    \end{tabular}
  }
\end{table}

\begin{table}[ht]
  \centering
  \caption{ImageNet-1k NearOOD Benchmark results on the ResNet50 architecture.}
  \label{tab:imagenet1k_benchmark/nearOOD}
  \resizebox{\textwidth}{!}{%
    \begin{tabular}{lcccccccccc}
      \toprule
      \multirow{2}{*}{\textbf{Method}}      &
      \multicolumn{2}{c}{\textbf{SSB-hard}} &
      \multicolumn{2}{c}{\textbf{NINCO}}    &
      \multicolumn{2}{c}{\textbf{Average}}                                                                                                                                                            \\
      \cmidrule(lr){2-3} \cmidrule(lr){4-5} \cmidrule(lr){6-7}
                                            & \textbf{FPR95 $\downarrow$} & \textbf{AUROC $\uparrow$} &
      \textbf{FPR95 $\downarrow$}           & \textbf{AUROC $\uparrow$}   &
      \textbf{FPR95 $\downarrow$}           & \textbf{AUROC $\uparrow$}                                                                                                                               \\
      \midrule
      OpenMax                               & $ 77.33 $                   & $ 71.37 $                 & $ 60.81 $             & $ 78.17 $             & $ 69.07 $             & $ 74.77 $             \\
      MSP                                   & $ 74.49 $                   & $ 72.09 $                 & $ 56.88 $             & $ 79.95 $             & $ 65.68 $             & $ 76.02 $             \\
      TempScale                             & $ \underline{73.9} $        & $ 72.87 $                 & $ 55.1 $              & $ 81.41 $             & $ 64.5 $              & $ 77.14 $             \\
      ODIN                                  & $ 76.83 $                   & $ 71.74 $                 & $ 68.16 $             & $ 77.77 $             & $ 72.5 $              & $ 74.75 $             \\
      MDS                                   & $ 92.1 $                    & $ 48.5 $                  & $ 78.8 $              & $ 62.38 $             & $ 85.45 $             & $ 55.44 $             \\
\bottomrule
    \end{tabular}
  }
\end{table}

\begin{table}[ht]
  \centering
  \caption{ImageNet-1k FarOOD benchmark results on the ViT-B/16 architecture.}
  \label{tab:imagenet1k_benchmark/ViT}
  \resizebox{\textwidth}{!}{%
    \begin{tabular}{lcccccccccc}
      \toprule
      \multirow{2}{*}{\textbf{Method}}         &
      \multicolumn{2}{c}{\textbf{iNaturalist}} &
      \multicolumn{2}{c}{\textbf{Textures}}    &
      \multicolumn{2}{c}{\textbf{OpenImage-O}} &
      \multicolumn{2}{c}{\textbf{Average}}                                                                                                                                                                                                           \\
      \cmidrule(lr){2-3} \cmidrule(lr){4-5} \cmidrule(lr){6-7} \cmidrule(lr){8-9} \cmidrule(lr){10-11}
                                               & \textbf{FPR95 $\downarrow$} & \textbf{AUROC $\uparrow$} &
      \textbf{FPR95 $\downarrow$}              & \textbf{AUROC $\uparrow$}   &
      \textbf{FPR95 $\downarrow$}              & \textbf{AUROC $\uparrow$}   &
      \textbf{FPR95 $\downarrow$}              & \textbf{AUROC $\uparrow$}                                                                                                                                                                           \\
      \midrule
      OpenMax                                  & $ 19.62 $                   & $ 94.93 $                 & $ 73.07 $             & $ 85.52 $             & $ 73.82 $             & $ 87.36 $             & $ 55.5 $              & $ 89.27 $         \\
      MSP                                      & $ 42.4 $                    & $ 88.19 $                 & $ 56.46 $             & $ 85.06 $             & $ 56.19 $             & $ 84.86 $             & $ 51.69 $             & $ 86.04 $         \\
      TempScale                                & $ 43.09 $                   & $ 88.54 $                 & $ 58.16 $             & $ 85.39 $             & $ 59.98 $             & $ 85.04 $             & $ 53.74 $             & $ 86.32 $         \\
\bottomrule
    \end{tabular}
  }
\end{table}

\begin{table}[ht]
  \centering
  \caption{ImageNet-1k NearOOD benchmark results on the ViT-B/16 architecture.}
  \label{tab:imagenet1k_benchmark/ViT/nearOOD}
  \resizebox{\textwidth}{!}{%
    \begin{tabular}{lcccccccccc}
      \toprule
      \multirow{2}{*}{\textbf{Method}}      &
      \multicolumn{2}{c}{\textbf{SSB-hard}} &
      \multicolumn{2}{c}{\textbf{NINCO}}    &
      \multicolumn{2}{c}{\textbf{Average}}                                                                                                                                                            \\
      \cmidrule(lr){2-3} \cmidrule(lr){4-5} \cmidrule(lr){6-7}
                                            & \textbf{FPR95 $\downarrow$} & \textbf{AUROC $\uparrow$} &
      \textbf{FPR95 $\downarrow$}           & \textbf{AUROC $\uparrow$}   &
      \textbf{FPR95 $\downarrow$}           & \textbf{AUROC $\uparrow$}                                                                                                                               \\
      \midrule
      OpenMax                               & $ 89.19 $                   & $ 68.6 $                  & $ 88.33 $             & $ 78.68 $             & $ 88.76 $             & $ 73.64 $             \\
      MSP                                   & $ 86.41 $                   & $ 68.94 $                 & $ 77.28 $             & $ 78.11 $             & $ 81.85 $             & $ 73.52 $             \\
      TempScale                             & $ 87.35 $                   & $ 68.55 $                 & $ 81.88 $             & $ 77.8 $              & $ 84.62 $             & $ 73.18 $             \\
      ODIN                                  & $ - $                       & $ - $                     & $ - $                 & $ - $                 & $ - $                 & $ - $                 \\
      MDS                                   & $ 83.47 $                   & $ 71.57 $                 & $ 48.77 $             & $ 86.52 $             & $ 66.12 $             & $ 79.04 $             \\
\bottomrule
    \end{tabular}
  }
\end{table}

\begin{table}[ht]
  \centering
  \caption{ImageNet-1k FarOOD benchmark results on the Swin-T architecture.}
  \label{tab:imagenet1k_benchmark/swin}
  \resizebox{\textwidth}{!}{%
    \begin{tabular}{lcccccccccc}
      \toprule
      \multirow{2}{*}{\textbf{Method}}         &
      \multicolumn{2}{c}{\textbf{iNaturalist}} &
      \multicolumn{2}{c}{\textbf{Textures}}    &
      \multicolumn{2}{c}{\textbf{OpenImage-O}} &
      \multicolumn{2}{c}{\textbf{Average}}                                                                                                                                                                                                               \\
      \cmidrule(lr){2-3} \cmidrule(lr){4-5} \cmidrule(lr){6-7} \cmidrule(lr){8-9} \cmidrule(lr){10-11}
                                               & \textbf{FPR95 $\downarrow$} & \textbf{AUROC $\uparrow$} &
      \textbf{FPR95 $\downarrow$}              & \textbf{AUROC $\uparrow$}   &
      \textbf{FPR95 $\downarrow$}              & \textbf{AUROC $\uparrow$}   &
      \textbf{FPR95 $\downarrow$}              & \textbf{AUROC $\uparrow$}                                                                                                                                                                               \\
      \midrule
      OpenMax                                  & $ 19.28 $                   & $ 95.08 $                 & $ 77.31 $             & $ 82.75 $             & $ 60.76 $             & $ 87.84 $             & $ 52.45 $             & $ 88.56 $             \\
      MSP                                      & $ 37.31 $                   & $ 89.86 $                 & $ 60.87 $             & $ 83.27 $             & $ 49.3 $              & $ 85.77 $             & $ 49.16 $             & $ 86.3 $              \\
      TempScale                                & $ 36.66 $                   & $ 90.37 $                 & $ 62.55 $             & $ 83.68 $             & $ 50.93 $             & $ 86.14 $             & $ 50.05 $             & $ 86.73 $             \\
\bottomrule
    \end{tabular}
  }
\end{table}

\begin{table}[ht]
  \centering
  \caption{ImageNet-1k NearOOD benchmark results on the Swin-T architecture.}
  \label{tab:imagenet1k_benchmark/swin/nearOOD}
  \resizebox{\textwidth}{!}{%
    \begin{tabular}{lcccccccccc}
      \toprule
      \multirow{2}{*}{\textbf{Method}}      &
      \multicolumn{2}{c}{\textbf{SSB-hard}} &
      \multicolumn{2}{c}{\textbf{NINCO}}    &
      \multicolumn{2}{c}{\textbf{Average}}                                                                                                                                                            \\
      \cmidrule(lr){2-3} \cmidrule(lr){4-5} \cmidrule(lr){6-7}
                                            & \textbf{FPR95 $\downarrow$} & \textbf{AUROC $\uparrow$} &
      \textbf{FPR95 $\downarrow$}           & \textbf{AUROC $\uparrow$}   &
      \textbf{FPR95 $\downarrow$}           & \textbf{AUROC $\uparrow$}                                                                                                                               \\
      \midrule
      OpenMax                               & $ 85.5 $                    & $ 71.56 $                 & $ 72.58 $             & $ 81.66 $             & $ 79.04 $             & $ 76.61 $             \\
      MSP                                   & $ 80.9 $                    & $ 71.78 $                 & $ 61.21 $             & $ 81.72 $             & $ 71.06 $             & $ 76.75 $             \\
      TempScale                             & $ 82.25 $                   & $ 71.86 $                 & $ 63.24 $             & $ 82.1 $              & $ 72.75 $             & $ 76.98 $             \\
      ODIN                                  & $ 86.45 $                   & $ 66.94 $                 & $ 86.05 $             & $ 69.35 $             & $ 86.25 $             & $ 68.15 $             \\
      MDS                                   & $ 83.6 $                    & $ 68.61 $                 & $ 53.71 $             & $ 81.75 $             & $ 68.65 $             & $ 75.18 $             \\
\bottomrule
    \end{tabular}
  }
\end{table}

\FloatBarrier

\end{document}